\documentclass[onefignum,onetabnum]{siamonline250211}

\usepackage{lipsum}
\usepackage{amsfonts}
\usepackage{graphicx}
\usepackage{epstopdf}
\usepackage{algorithmic}
\ifpdf
  \DeclareGraphicsExtensions{.eps,.pdf,.png,.jpg}
\else
  \DeclareGraphicsExtensions{.eps}
\fi

\usepackage{enumitem}
\setlist[enumerate]{leftmargin=.5in}
\setlist[itemize]{leftmargin=.5in}

\newsiamremark{remark}{Remark}
\newsiamremark{hypothesis}{Hypothesis}
\crefname{hypothesis}{Hypothesis}{Hypotheses}
\newsiamthm{claim}{Claim}
\newsiamremark{fact}{Fact}
\crefname{fact}{Fact}{Facts}

\headers{PnP Volumetric Reconstruction for CS-LSM}{J. Jia, Y. Gong, X. Zhang, J. Chai, Y. Ding, and Y. Lou}

\title{Plug-and-Play Volumetric Reconstruction for Compressive Sensing Light-Sheet Microscopy\thanks{Submitted to the editors DATE.
\funding{J.J. was partially supported by AMS-Simons Travel Grant 330934. Y.D. was partially supported by NIH grants R00HL148493 and R01HL162635, NSF grant 2503230, and the Cecil H. and Ida Green Professorship in Systems Biology Science. Y.L. was partially supported by NSF CAREER award DMS-2414705.}}}

\author{
Jianqing Jia\thanks{School of Data and Information Sciences, The University of North Carolina at Chapel Hill, Chapel Hill, NC 27599, USA (\email{jqjia@unc.edu}).}
\and
Yi Gong\thanks{Department of Mathematics, The University of North Carolina at Chapel Hill, Chapel Hill, NC 27599, USA (\email{yigong@unc.edu}). Jia and Gong contributed equally to this work.}
\and
Xinyuan Zhang\thanks{Department of Bioengineering, The University of Texas at Dallas, Richardson, TX 75080, USA (\email{xinyuan.zhang@utdallas.edu}, \email{jichen.chai@utdallas.edu}, \email{yichen.ding@utdallas.edu}).}
\and
Jichen Chai\footnotemark[4]
\and
Yichen Ding\footnotemark[4]
\and
Yifei Lou\thanks{Department of Mathematics, School of Data and Information Sciences, The University of North Carolina at Chapel Hill, Chapel Hill, NC 27599, USA (\email{yflou@unc.edu}).}
}

\usepackage{cite}
\usepackage{amsmath,amssymb,amsfonts}         
\usepackage{algorithm}
\usepackage{stfloats}
\usepackage{multirow}
\usepackage{bm}
\usepackage{array}
\usepackage{graphicx}
\usepackage{textcomp}
\usepackage{xcolor}
\usepackage{array}
\usepackage{pgfplots}
\usepackage{tikz}
\pgfplotsset{compat=newest}
\def\BibTeX{{\rm B\kern-.05em{\sc i\kern-.025em b}\kern-.08em
    T\kern-.1667em\lower.7ex\hbox{E}\kern-.125emX}}

\usepackage[normalem]{ulem}

\newtheorem{assumption}{Assumption}[section]

\usepackage{amsopn}

\usepackage{subcaption}

\ifpdf
\hypersetup{
  pdftitle={Plug-and-Play Volumetric Reconstruction for Compressive Sensing Light-Sheet Microscopy},
  pdfauthor={Y. Gong, X. Zhang, J. Chai, Y. Ding, and Y. Lou}
}
\fi

\begin{document}

\maketitle

\begin{abstract} 
We investigate volumetric reconstruction for compressive sensing light-sheet microscopy (CS-LSM), where fast volumetric imaging is achieved by encoding multiple axial planes into each camera exposure.
To recover the underlying volume from highly multiplexed measurements, we propose a plug-and-play (PnP) framework that flexibly incorporates any user-specified denoiser into the reconstruction process. Building on a slice-based formulation, we further introduce an axial-coupled model that exploits correlations between adjacent slices to improve volumetric continuity. For efficient computation, we derive a Woodbury-based update for the data-consistency step in both the slice-based and axial-coupled formulations, and employ a Gauss–Seidel sweep for the denoising step in the axial-coupled model. 
Under a weakly convex regularization assumption, we establish subsequential convergence of the proposed algorithm.
Experiments on synthetic and real zebrafish-heart data demonstrate that the proposed framework successfully recovers cellular structures from compressed measurements,  and provide practical insights into the comparative performance of commonly used denoisers within the PnP framework under the CS-LSM setup.
\end{abstract}



\section{Introduction}
Biological processes such as cardiac contraction and intracardiac flow evolve in three-dimensional (3D) space and over extremely short time scales~\cite{bers2002cardiac}.
Capturing these dynamics \textit{in vivo}  requires volumetric imaging methods that achieve cellular resolution over organ-scale fields of view while maintaining sufficiently high temporal resolution~\cite{chen2019diamond}. At the same time, light exposure must be carefully controlled to limit photobleaching and phototoxicity~\cite{zhang2022computational}, as these measurements are often performed in living specimens.

In fluorescence microscopy, volumetric imaging speed, spatial resolution, field of view, and excitation burden are constrained by fundamental trade-offs, thereby hindering their simultaneous achievement. Existing 3D imaging methods for beating hearts include confocal laser scanning microscopy (CLSM) \cite{Minsky1988AO,White1987JCB,liebling2005four}, two-photon microscopy (TPM) \cite{Denk1990Science,li2012intravital,mahou2014multicolor}, and light-sheet microscopy (LSM) \cite{Huisken2004Science,Santi2011JHC,mickoleit2014high,lee20164,taylor2019adaptive,zhang20234d}. CLSM and TPM offer strong optical sectioning, but their point- or line-scanning acquisition limits the rate at which full volumes can be recorded and may lead to substantial photobleaching, phototoxicity, or thermal load during prolonged high-speed imaging \cite{Icha2017Bioessays,Helmchen2005NatMethods}. In contrast, LSM illuminates only the imaging plane and detects fluorescence from an orthogonal direction, positioning it as an attractive approach for fast volumetric imaging with reduced light dose. 
Despite these advantages, conventional LSM struggles to capture dynamics that evolve on millisecond timescales across 3D volumes, and retrospective synchronization or prospective optical gating is often employed to compensate \cite{mickoleit2014high,lee20164,taylor2019adaptive,zhang20234d}. 
Approaches such as rapid beam steering \cite{voleti2019real} or multiplane recording~\cite{weber2023high} partially mitigate this limitation but typically rely on specialized optical instrumentation and detection hardware.
More fundamentally, detector bandwidth imposes hard limits on volumetric imaging performance, as higher acquisition rates come at the expense of spatial sampling, field of view, or signal quality.

This trade-off motivates coded and compressive acquisition strategies, and has recently driven the integration of LSM with compressive sensing (CS)  \cite{donoho2006compressed,candes2007sparsity} for more efficient volumetric encoding.
While such coded and compressive volumetric imaging approaches show clear promise for improving volumetric imaging speed and reducing the number of required measurements, achieving fast \textit{in vivo} volumetric imaging at cellular resolution remains challenging. 
For instance, a spatially modulated LSM recover volumes from patterned illumination using CS, but it relies on multiple sequential acquisitions, limiting volumetric speed \cite{calisesi2019spatially}.  Snapshot temporal compressive LSM can recover multiple frames from a single measurement, but it is typically designed for a single light-sheet plane rather than a full 3D volume \cite{wang2023snapshot}. In addition, light field microscopy enables single-shot volumetric capture by encoding angular information onto a single sensor; however, the associated trade-off between angular and spatial sampling often reduces effective resolution, introduces on-focal artifacts, or requires customized optical components and extensive training data for neural network-based reconstruction, making cellular-scale \textit{in vivo} imaging challenging \cite{saberigarakani2025volumetric,Wang2021VCDLFM,Wang2021Hybrid,wang2025kilohertz}. Taken together, these studies 
indicate that fast acquisition alone is insufficient; when measurements are strongly multiplexed and undersampled, effective reconstruction becomes a central component of the imaging pipeline.

Addressing the acquisition side of this pipeline, prior work by our team developed a CS-LSM platform for compressed light-sheet acquisition \cite{zhang2025instantaneous}. This paper focuses on the reconstruction problem; for details on hardware design, calibration, and acquisition protocol, we refer the reader to \cite{zhang2026compressive}.
In the CS-LSM system, axial scanning is synchronized with spatial light modulation via a digital micromirror device (DMD). Following the incoherent measurement framework of CS \cite{candes2007sparsity}, random binary masks are used to encode fluorescence from multiple depth planes into a single exposure, enabling volumetric imaging at rates of 200 volumes per second while maintaining cellular resolution \cite{zhang2025instantaneous}. It also reduces the number of recorded measurements and the associated storage requirements relative to non-compressed acquisition, thereby improving data efficiency in high-speed volumetric imaging~\cite{zhang2026compressive}. 

While this hardware platform enables high-speed compressed acquisition, the resulting measurements are highly multiplexed and undersampled, necessitating robust and efficient reconstruction algorithms to recover high-quality volumetric reconstructions.
The main objective of this paper is therefore to develop a flexible reconstruction framework tailored to this hardware system. 
Although motivated by the present CS-LSM platform, the resulting methodology may be applicable to a broader class of coded, multiplexed, and computational volumetric imaging settings. 

To this end, we build on a plug-and-play (PnP) framework \cite{venkatakrishnan2013plug} and solve the model using the alternating direction method of multipliers (ADMM) \cite{boyd2011distributed}, which we refer to as PnP-ADMM. We first formulate a slice-based reconstruction model, in which each slice is reconstructed independently. This design exploits the diagonal structure of the DMD masks, enabling efficient data-consistency updates via Woodbury-based inversion. To better preserve volumetric continuity, we subsequently introduce an axially coupled reconstruction model that enforces smoothness between adjacent slices along the $z$ direction. This axial coupling is designed to leverage inter-slice correlations within the reconstructed volume, rather than temporal correlations across different cardiac phases. In contrast to full 3D regularization or volumetric patch-based reconstruction methods \cite{persson2001total,maggioni2012nonlocal,jia2012four}, we incorporate axial coupling within the denoising step using a Gauss–Seidel sweep. This approach is computationally efficient and, importantly, enables us to establish convergence guarantees for the proposed algorithm.

The name \textit{plug-and-play} reflects the fact that one subproblem is formulated as a denoising step, allowing a wide range of denoisers to be incorporated without altering the overall reconstruction framework. This flexibility has also been leveraged in hyperspectral unmixing \cite{zhao2021plug}, limited-angle tomography~\cite{zhao2025limited}, and autonomous driving \cite{xiong2026prune2drive}. Here, we investigate the performance of classical denoisers, such as Tikhonov \cite{tikhonov1977solutions}, total variation (TV) \cite{rudin1992nonlinear,chambolle1997image,chambolle2004algorithm,goldstein2009split}, and block-matching and 3D filtering (BM3D) \cite{dabov2007image}, as well as deep learning (DL)-based denoisers, including the denoising convolutional neural network (DnCNN) \cite{zhang2017beyond}, the fast and flexible denoising network (FFDNet) \cite{zhang2018ffdnet}, and the deep residual U-Net denoiser (DRUNet) \cite{zhang2021plug}. Experiments on both synthetic and real zebrafish-heart data demonstrate the effectiveness of the proposed framework and provide practical insights into the strengths and limitations of these off‑the‑shelf denoisers.

For convex regularizers such as Tikhonov and TV, the convergence of PnP-ADMM follows directly from classical ADMM theories \cite{boyd2011distributed,eckstein2015understanding}. However, this guarantee does not readily extend to more advanced denoisers, such as BM3D and DL-based methods, for which an explicit underlying prior is typically unavailable, let alone one that can be verified to be convex. To address this gap, Chan et al.~\cite{chan2017plug} established convergence of PnP algorithms under the assumption that the denoiser is nonexpansive, while Ryu et al.~\cite{ryu2019plug} relaxed this requirement by imposing nonexpansiveness on the residual operator, at the cost of requiring strong convexity of the data-fidelity term, which is not satisfied in the CS-LSM setting. In this work, we establish subsequential convergence of the PnP-ADMM algorithm for the proposed axial-coupled model for a weakly convex regularization term. 

The main contributions of this paper are summarized as follows:
\begin{itemize}
  \item We introduce an axial-coupled reconstruction model that leverages inter-slice correlations along the $z$ direction, leading to improved reconstruction performance compared with slice-based model. 
  \item We tailor the PnP-ADMM framework to the CS-LSM setup through efficient algorithmic updates, including a Woodbury-based reformulation of the linear subproblem and a Gauss–Seidel update strategy for the axial-coupled model.
  \item We analyze the convergence properties of the proposed algorithm, establishing subsequential convergence under the assumption that the image prior is weakly convex.
  \item Experiments on synthetic and real data provide practical guidance for selecting off-the-shelf denoisers in the CS-LSM reconstruction problem.
\end{itemize}

The rest of this paper is organized as follows. Section~\ref{sec:method} introduces the CS-LSM forward model and the proposed PnP-ADMM reconstruction framework, including both slice-based and axial-coupled models. Section~\ref{sec:convergence} discusses convergence properties of the proposed framework. Section~\ref{sec:experiments} presents numerical results on synthetic and real datasets. Finally, Section~\ref{sec:conclusion} concludes the paper and outlines future work.

\section{Proposed Approaches}\label{sec:method}

In this section, we first introduce the CS-LSM forward model in Section~\ref{sect:forward}. We then present a slice-based reconstruction model in Section~\ref{sect:slice}, where a 3D volume is recovered through a collection of 2D subproblems. Next, in Section~\ref{sect:axial}, we extend this formulation to an axial-coupled model by enforcing inter-slice smoothness along the $z$ direction. Both models are solved within an ADMM framework. Notably, the regularization subproblem is interpreted as a denoising step,  reflecting the \textit{plug-and-play} nature of the framework.

\begin{figure*}
\centering
\includegraphics[width=0.95\columnwidth]{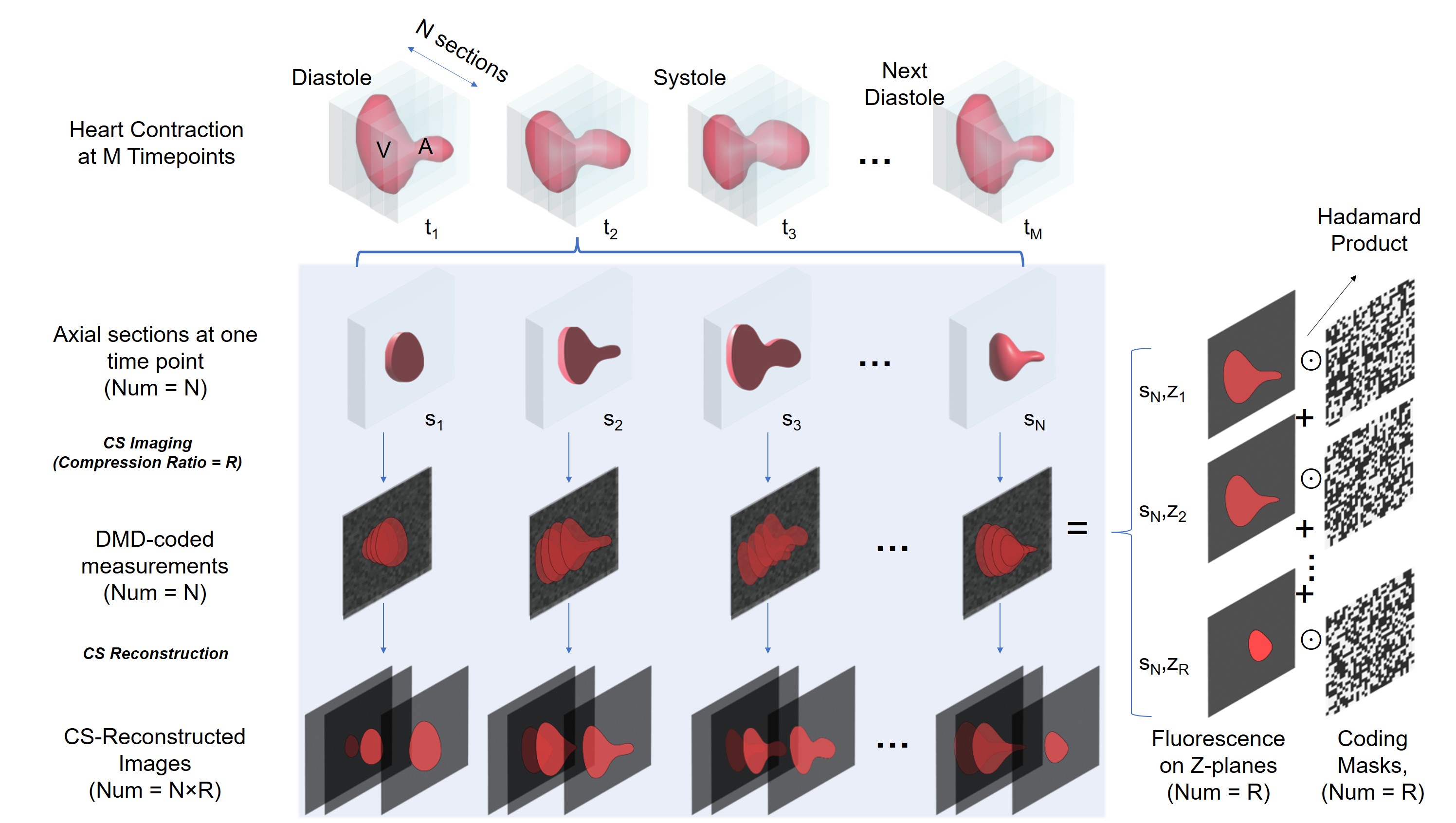}
\caption{Illustration of the CS-LSM image formation and reconstruction process. Cardiac contraction is observed over $M$ time points during one heartbeat. At one fixed time point, the 3D volume is sampled through $N$ compressed camera shots. In each shot, fluorescence signals from $R$ axial planes are encoded by distinct binary masks and summed into a single DMD-coded measurement. From the resulting $N$ compressed measurements, the CS reconstruction recovers $N\times R$ axial slices.}
\label{fig:forward_model}
\end{figure*}

\subsection{Forward model and notation}\label{sect:forward}

We adopt the CS-LSM acquisition model from \cite{zhang2025instantaneous} and focus on the reconstruction problem. Specifically, we obtain compressed image data through a set of binary masks, as illustrated in Fig.~\ref{fig:forward_model}. The top row of Fig.~\ref{fig:forward_model} shows a sequence of 3D cardiac volumes acquired at $M$ distinct time points during one heartbeat, indexed by $t_1,t_2,\dots,t_M$. In this paper, we focus on reconstruction at a fixed time point, recovering one single 3D volume  at a time. Extending this across time yields a 4D sequence, while explicit temporal coupling is left for future work. At a time point, the underlying 3D volume is acquired through $N$ compressed camera shots over one scanning period. As shown in Fig.~\ref{fig:forward_model}, each shot corresponds to one axial block of the volume, and each axial block contains $R$ axial slices along the $z$ direction, where $R$ is the compression ratio, namely, the number of axial slices encoded within a single camera exposure. Consequently, the full reconstructed volume contains $S:=NR$ axial slices in total. We denote these unknown slices by $\{\bm v_n\}_{n=1}^{S}$, where each $\bm v_n\in\mathbb R^{p}$ represents one vectorized 2D slice\footnote{This vectorization is introduced only for notational convenience; in implementation, each slice is kept in its original 2D form.} and $p$ denotes the total number of pixels in each 2D slice. 

In the CS acquisition process, the DMD applies $R$ binary masks sequentially during each camera exposure. Within the $j$th shot ($j=1, \cdots, N$), the $R$ axial slices in the corresponding section group are modulated by distinct binary masks through element-wise (Hadamard) products and then summed on the detector, producing a single compressed 2D measurement, which we vectorize as $\bm b_j\in\mathbb R^{p}$. Let $\bm\phi_r\in\mathbb R^{p\times p}$ denote the diagonal masking operator for the $r$th slice ($r=1, \cdots, R)$, so that the Hadamard modulation can be written as a matrix-vector product. Then the $j$th compressed measurement can be mathematically modeled as
\begin{equation}\label{eq:forward_model}
\bm b_j=\sum_{r=1}^{R}\bm\phi_r \bm v_{(j-1)R+r}+\bm\eta_j,\qquad j=1,\dots,N,
\end{equation}
where $\bm\eta_j$ denotes measurement noise. 
Thus, the reconstruction task is to recover the full stack of $S$ axial slices $\{\bm v_n\}_{n=1}^{S}$ from only $N$ compressed measurements stored in $\{\bm b_j\}_{j=1}^{N}$, which is a highly underdetermined inverse problem. Moreover, since adjacent axial slices in the volume are often strongly correlated, it is beneficial to exploit such inter-slice structure during reconstruction, which motivates the axial-coupled model introduced in Section~\ref{sect:axial}.

\subsection{Slice-based reconstruction}\label{sect:slice}

We begin with a slice-based model, in which each axial slice is regularized individually. Assuming that the measurement noise can be well approximated as additive Gaussian\footnote{The Gaussian noise model is a common approximation that is well-justified at sufficiently high photon counts, which is consistent with our hardware setting.}, we employ a least-squares data-fidelity term and consider the following optimization problem, 
\begin{equation}\label{eq:model-slice}
\min_{\{\bm v_n\}_{n=1}^{S}}
\frac{1}{2}\sum_{j=1}^{N}\left\|\bm b_j-\sum_{r=1}^{R}\bm\phi_r \bm v_{(j-1)R+r}\right\|^2
+\lambda\sum_{n=1}^{S}\psi(\bm v_n),
\end{equation}
where $\psi(\cdot)$ denotes a slice-wise regularization term and $\lambda>0$ is a weighting parameter.

To solve \eqref{eq:model-slice}, we introduce auxiliary variables $\{\bm u_n\}_{n=1}^{S}$ and rewrite the model as
\begin{equation}\label{eq:model-slice-split}
\begin{aligned}
\min_{\{\bm v_n,\bm u_n\}}
\quad &
\frac{1}{2}\sum_{j=1}^{N}\left\|\bm b_j-\sum_{r=1}^{R}\bm\phi_r \bm v_{(j-1)R+r}\right\|^2
+\lambda\sum_{n=1}^{S}\psi(\bm u_n) \\
\text{subject to}\quad & \bm v_n=\bm u_n,\qquad n=1,\dots,S.
\end{aligned}
\end{equation}
The corresponding augmented Lagrangian is 
\begin{align}
\mathcal{L}_{\rho}(\{\bm u_n\},\{\bm v_n\},\{\bm d_n\})
:=&\;
\frac{1}{2}\sum_{j=1}^{N}\left\|\bm b_j-\sum_{r=1}^{R}\bm\phi_r \bm v_{(j-1)R+r}\right\|^2\notag \\
&+\sum_{n=1}^{S}\left(
\lambda\psi(\bm u_n)
+\frac{\rho}{2}\|\bm v_n-\bm u_n+\bm d_n\|^2
-\frac{\rho}{2}\|\bm d_n\|^2
\right),
\end{align}
where $\{\bm d_n\}_{n=1}^{S}$ are scaled dual variables and $\rho>0$ is the penalty parameter. 
The ADMM iterations are given by
\begin{equation}\label{eq:ADMM}
\left\{
\begin{aligned}
\{\bm u_n^{k+1}\}
&=
\arg\min_{\{\bm u_n\}}
\mathcal{L}_{\rho}(\{\bm u_n\},\{\bm v_n^k\},\{\bm d_n^k\}),\\
\{\bm v_n^{k+1}\}
&=
\arg\min_{\{\bm v_n\}}
\mathcal{L}_{\rho}(\{\bm u_n^{k+1}\},\{\bm v_n\},\{\bm d_n^k\}),\\
\bm d_n^{k+1}
&=
\bm d_n^k+\bm v_n^{k+1}-\bm u_n^{k+1},
\end{aligned}
\right.
\end{equation}
where $n=1,\dots,S$ and the superscript $k$ counts the iteration number.

The $\bm u$-subproblem decouples across slices:
\begin{equation}\label{eq:u-sub}
\bm u_n^{k+1}
=
\arg\min_{\bm z}\;
\lambda\psi(\bm z)+\frac{\rho}{2}\|\bm z-(\bm v_n^{k}+\bm d_n^k)\|^2,
\qquad n=1,\dots,S.
\end{equation}
From the PnP perspective, this subproblem is interpreted as a denoising step applied to the input
$\bm v_n^{k}+\bm d_n^k$. Many existing denoising approaches, or denoisers for short, can be formulated as
\begin{equation}\label{eq:denoiser}
    \mathcal K_\sigma(\bm f) := \arg\min_{\bm z} \psi(\bm z)+ \frac 1{2\sigma^2}\|\bm z-\bm f\|^2,
\end{equation}
where $\sigma>0$ is an input noise-level parameter common to methods such as BM3D, DnCNN, FFDNet, and DRUNet, and typically corresponds to the standard deviation of the noise. 
Dividing \eqref{eq:u-sub} by $\rho$ and comparing with the denoising formulation \eqref{eq:denoiser}, we identify the denoiser noise level as
$\sigma_\lambda=\sqrt{\lambda/\rho}$ and adopt the following update:
\begin{equation}\label{eq:u-pnp-slice}
\bm u_n^{k+1}=\mathcal K_{\sigma_\lambda}(\bm v_n^{k}+\bm d_n^k),
\end{equation}
where $\mathcal K_{\sigma_\lambda}$ denotes the selected denoiser with its internal noise level $\sigma_\lambda$ controlled by $\lambda.$
 When $\mathcal K_{\sigma_\lambda}$ is induced by an explicit regularizer, \eqref{eq:u-sub} reduces to a proximal update. As an example, Tikhonov regularization with $\psi(\bm z)=\|\bm z\|^2$ yields  
\begin{equation}\label{eq:tikhonov-slice}
\bm u_n^{k+1}
=
\frac{\rho(\bm v_n^{k}+\bm d_n^k)}{2\lambda+\rho}.
\end{equation}

For the $\bm v$-subproblem, the variables associated with each compressed shot can be updated shotwise. Fix $j\in\{1,\dots,N\}$ and define
\[
\bm c_r^{(j),k+1}
:=
\bm u_{(j-1)R+r}^{k+1}
-
\bm d_{(j-1)R+r}^{k},
\qquad r=1,\dots,R.
\]
Let
\[
\bm v^{(j)}
:=
[\bm v_{(j-1)R+1}^\top,\dots,\bm v_{jR}^\top]^\top
\in\mathbb R^{pR},
\qquad
\bm c^{(j),k+1}
:=
[(\bm c_1^{(j),k+1})^\top,\dots,(\bm c_R^{(j),k+1})^\top]^\top
\in\mathbb R^{pR},
\]
and
\[
\Phi:=[\bm\phi_1,\dots,\bm\phi_R]\in\mathbb R^{p\times pR}.
\]
Then the corresponding subproblem becomes
\begin{equation}\label{eq:v-sub-1}
\min_{\bm v^{(j)}}\;
\frac{1}{2}\|\bm b_j-\Phi \bm v^{(j)}\|^2+\frac{\rho}{2}\|\bm v^{(j)}-\bm c^{(j),k+1}\|^2.
\end{equation}
Its closed-form solution is
\begin{equation}\label{eq:v-update}
\bm v^{(j),k+1}
=
(\Phi^{\top}\Phi+\rho I_{pR})^{-1}
(\Phi^{\top}\bm b_j+\rho\bm c^{(j),k+1}).
\end{equation}

The matrix inverse in \eqref{eq:v-update} can be simplified by the Woodbury matrix identity
\begin{equation}\label{Woodbury}
(I+UV)^{-1}=I-U(I+VU)^{-1}V.
\end{equation} 
Applying \eqref{Woodbury} yields the equivalent update
\begin{equation}\label{eq:v-update-fast}
\bm v^{(j),k+1}
=
\bm c^{(j),k+1}
+\Phi^{\top}(\rho I_p+\Phi\Phi^{\top})^{-1}
(\bm b_j-\Phi\bm c^{(j),k+1}).
\end{equation}
Equation \eqref{eq:v-update-fast} is computationally attractive because it replaces the inversion of the $pR\times pR$ matrix in \eqref{eq:v-update} by the inversion of the $p\times p$ matrix $\rho I_p+\Phi\Phi^{\top}$, which has a simple structure. Indeed, since each $\bm\phi_r$ is diagonal, then the product $\bm\phi_r\bm\phi_r^{\top}$ is diagonal and
$
\Phi\Phi^{\top}=\sum_{r=1}^R \bm\phi_r\bm\phi_r^{\top}
$
 is diagonal as well. Therefore, $(\rho I_p+\Phi\Phi^{\top})^{-1}$ can be computed element-wise. More specifically, if $\phi_r(i)$ denotes the $i$th diagonal entry of $\bm\phi_r$, then
\begin{equation}\label{eq:phi-diag}
[\Phi\Phi^{\top}]_{ii}=\sum_{r=1}^R \phi_r(i)^2.
\end{equation}
For binary masks, $\phi_r(i)^2=\phi_r(i)$, so the quantity in \eqref{eq:phi-diag} is simply the number of masks that are active at pixel $i$. Therefore, the data-consistency step in \eqref{eq:v-update-fast} does not require any matrix inversion and can instead be implemented using inexpensive pointwise operations. 
Specifically, for a fixed shot $j$, \eqref{eq:v-update-fast} admits a pixel-wise interpretation: for each detector pixel location $i$, define the $R$-dimensional vectors
\[
\bm m(i)
:=
[\phi_1(i),\dots,\phi_R(i)]^{\top}\in\mathbb R^R,
\qquad
\widehat{\bm c}^{(j),k+1}(i)
:=
[c_1^{(j),k+1}(i),\dots,c_R^{(j),k+1}(i)]^{\top}\in\mathbb R^R,
\]
and let
\[
\widehat{\bm v}^{(j),k+1}(i)
:=
[v_{(j-1)R+1}^{k+1}(i),\dots,v_{jR}^{k+1}(i)]^{\top}\in\mathbb R^R,
\]
where $\widehat{\bm c}^{(j),k+1}(i)$ and $\widehat{\bm v}^{(j),k+1}(i)$ collect the current and updated values of the $R$ slices at pixel location $i$ within the $j$th shot, respectively. Then \eqref{eq:v-update-fast} can be written as
\begin{equation}\label{eq:v-update-pixel}
\widehat{\bm v}^{(j),k+1}(i)
=
\widehat{\bm c}^{(j),k+1}(i)
+
\bm m(i)
\frac{
b_j(i)-\bm m(i)^{\top}\widehat{\bm c}^{(j),k+1}(i)
}{
\rho+\|\bm m(i)\|_2^2
}, \quad i=1, \cdots, p.
\end{equation}
This expression shows that the data-consistency step decouples across spatial pixels: different detector pixels can be updated independently, while the coupling occurs only among the $R$ slices sharing the same detector location. In other words, \eqref{eq:v-update-pixel} is a residual correction along the local mask direction $\bm m(i)$, scaled by the normalization factor $\rho+\|\bm m(i)\|_2^2$.

\subsection{Axial-coupled reconstruction}\label{sect:axial}

While the slice-based model treats each axial slice independently, adjacent slices in a 3D volume are often strongly correlated. Ignoring this inter-slice structure may lead to discontinuities or inconsistent reconstructions along the $z$ direction. A natural remedy is to apply full 3D regularization directly to the reconstructed volume, such as 3D TV \cite{persson2001total}, BM4D \cite{maggioni2012nonlocal} (a volumetric extension of BM3D), or temporal non-local means \cite{jia2012four}. 
We incorporate a quadratic axial coupling term that penalizes differences between neighboring slices.  This pairwise coupling admits an efficient Gauss--Seidel (GS) update strategy for the denoising step, for which we establish subsequential convergence in Section \ref{sec:convergence}. Concretely, the axial-coupled model reads
\begin{equation}\label{eq:model-axial}
\min_{\{\bm v_n\}_{n=1}^{S}}
\frac{1}{2}\sum_{j=1}^{N}\left\|\bm b_j-\sum_{r=1}^{R}\bm\phi_r \bm v_{(j-1)R+r}\right\|^2
+\lambda\sum_{n=1}^{S}\psi(\bm v_n)
+\frac{\gamma}{2}\sum_{n=1}^{S}\|\bm v_n-\bm v_{n-1}\|^2,
\end{equation}
where $\gamma>0$ controls the strength of axial coupling. For notational convenience, we impose periodic boundary conditions, $\bm v_0=\bm v_S$.

As in the slice-based case, we introduce auxiliary variables $\{\bm u_n\}_{n=1}^{S}$ and rewrite \eqref{eq:model-axial} as
\begin{equation}\label{eq:model-axial-split}
\begin{aligned}
\min_{\{\bm v_n,\bm u_n\}}
\quad &
\frac{1}{2}\sum_{j=1}^{N}\left\|\bm b_j-\sum_{r=1}^{R}\bm\phi_r \bm v_{(j-1)R+r}\right\|^2
+\lambda\sum_{n=1}^{S}\psi(\bm u_n)
+\frac{\gamma}{2}\sum_{n=1}^{S}\|\bm u_n-\bm u_{n-1}\|^2 \\
\text{subject to}\quad & \bm v_n=\bm u_n,\qquad n=1,\dots,S.
\end{aligned}
\end{equation}
The corresponding augmented Lagrangian is
\begin{align}
\mathcal{L}_{\rho}(\{\bm u_n\},\{\bm v_n\},\{\bm d_n\})
:=&\;
\frac{1}{2}\sum_{j=1}^{N}\left\|\bm b_j-\sum_{r=1}^{R}\bm\phi_r \bm v_{(j-1)R+r}\right\|^2 \notag\\
+\sum_{n=1}^{S}&\left(
\lambda\psi(\bm u_n)
+\frac{\rho}{2}\|\bm v_n-\bm u_n+\bm d_n\|^2
-\frac{\rho}{2}\|\bm d_n\|^2
\right)
+\frac{\gamma}{2}\sum_{n=1}^{S}\|\bm u_n-\bm u_{n-1}\|^2,
\label{eq:aug-axial}
\end{align}
where $\{\bm d_n\}_{n=1}^{S}$ are scaled dual variables and $\rho>0$ is the penalty parameter.

Following the ADMM scheme \eqref{eq:ADMM}, we first update the auxiliary variables $\{\bm u_n\}_{n=1}^{S}$ with $\{\bm v_n^k\}$ and $\{\bm d_n^k\}$ fixed. Due to the axial coupling term, the $\bm u$-subproblem no longer decouples across slices:
\begin{equation}\label{eq:u-sub-global-axial}
\{\bm u^{k+1}_n\}_{n=1}^{S}=\arg\min_{\{\bm z_n\}_{n=1}^{S}}
\lambda\sum_{n=1}^{S}\psi(\bm z_n)
+\frac{\rho}{2}\sum_{n=1}^{S}\|\bm z_n-(\bm v_n^{k}+\bm d_n^k)\|^2
+\frac{\gamma}{2}\sum_{n=1}^{S}\|\bm z_n-\bm z_{n-1}\|^2,
\end{equation}
that is, each $\bm u_n$ depends on its neighboring slices $\bm u_{n-1}$ and $\bm u_{n+1}.$
We adopt a GS strategy, updating one slice at a time in a cyclic sweep with periodic boundary conditions, i.e., $\bm u_0=\bm u_S$ and $\bm u_{S+1}=\bm u_1$, where each slice is updated with its two neighbors fixed at their most recently computed values. More precisely, for $n=1,\dots,S$, define 
\begin{equation}\label{eq:u-periodic}
    \bm u_{n,-}^{k}
:=
\begin{cases}
\bm u_{n-1}^{k+1}, & n=2,\dots,S,\\
\bm u_S^{k}, & n=1,
\end{cases}
\qquad
\bm u_{n,+}^{k}
:=
\begin{cases}
\bm u_{n+1}^{k}, & n=1,\dots,S-1,\\
\bm u_1^{k+1}, & n=S.
\end{cases}
\end{equation}
Then each slice is updated by solving
\begin{equation}
\bm u_n^{k+1}
=
\arg\min_{\bm z}\;
\lambda\psi(\bm z)
+\frac{\rho}{2}\|\bm z-(\bm v_n^{k}+\bm d_n^k)\|^2 
+\frac{\gamma}{2}\|\bm z-\bm u_{n,-}^{k}\|^2
+\frac{\gamma}{2}\|\bm z-\bm u_{n,+}^{k}\|^2 .
\label{eq:axial_u_sub}
\end{equation}

Completing the square and defining 
\begin{equation}\label{eq:g-update}
\bm g_n^k
=
\frac{\rho(\bm v_n^{k}+\bm d_n^k)+\gamma(\bm u_{n,-}^{k}+\bm u_{n,+}^{k})}{\rho+2\gamma}\in\mathbb R^p,
\end{equation}
the subproblem \eqref{eq:axial_u_sub} reduces to
\begin{equation}\label{eq:u-update-axial}
\bm u_n^{k+1}
=
\arg\min_{\bm z}\;
\lambda\psi(\bm z)+\frac{\rho+2\gamma}{2}\|\bm z-\bm g_n^k\|^2.
\end{equation}
From the PnP perspective, \eqref{eq:u-update-axial} can be interpreted as a denoising step applied to the axial-coupled quantity $\bm g_n^k$. Normalizing \eqref{eq:u-update-axial} by $\rho+2\gamma$ and comparing with \eqref{eq:denoiser} leads to the denoising update:
\begin{equation}\label{eq:u-pnp-axial}
\bm u_n^{k+1}
=
\mathcal K_{\sigma_\lambda}(\bm g_n^k),
\end{equation}
with $\sigma_\lambda=\sqrt{\frac{\lambda}{\rho+2\gamma}}$.
When $\mathcal K_{\sigma_\lambda}$ is induced by an explicit regularizer, for example, $\psi(\bm z)=\|\bm z\|^2$, the Tikhonov update can be written as the closed-form update
\begin{equation}\label{eq:tikhonov-axial}
\bm u_n^{k+1}
=
\frac{
\rho(\bm v_n^{k}+\bm d_n^k)
+\gamma(\bm u_{n,-}^{k}+\bm u_{n,+}^{k})
}{2\lambda+\rho+2\gamma}.
\end{equation}

With $\{\bm u_n^{k+1}\}_{n=1}^{S}$ fixed, the $\bm v$- and $\bm d$-updates are the same as in the slice-based model. We summarize the axial-coupled PnP-ADMM scheme in Algorithm~\ref{alg:cs-lsm-pnp}. When $\gamma=0$, the algorithm reduces to the slice-based model.

\begin{algorithm}
\caption{Axial-Coupled CS-LSM Reconstruction via PnP-ADMM}
\label{alg:cs-lsm-pnp}
\begin{algorithmic}
\STATE{Input: measurements $\{\bm b_j\}_{j=1}^{N}$, masks $\{\bm\phi_r\}_{r=1}^{R}$, parameters $\lambda,\rho,\gamma>0$, and denoiser $\mathcal K_{\sigma_\lambda}$}
\STATE{Define $S:=NR$ and initialize $\{\bm u_n^0,\bm v_n^0,\bm d_n^0\}_{n=1}^{S}$} 
\STATE{Set $k:=0$}
\WHILE{the stopping criterion is not satisfied}
    \FOR{$n=1,\dots,S$}
        \STATE{Compute $\bm g_n^k$ by \eqref{eq:g-update}}
        \STATE{Update $\bm u_n^{k+1}$ by \eqref{eq:u-pnp-axial}}
    \ENDFOR
    \FOR{$j=1,\dots,N$}
        \STATE{Collect the slice indices $n=(j-1)R+1,\dots,jR$}
        \STATE{Update $\{\bm v_n^{k+1}\}_{n=(j-1)R+1}^{jR}$ pixelwise via \eqref{eq:v-update-pixel}}
    \ENDFOR
    \FOR{$n=1,\dots,S$}
        \STATE{Update $\bm d_n^{k+1}=\bm d_n^k+\bm v_n^{k+1}-\bm u_n^{k+1}$}
    \ENDFOR
    \STATE{Set $k:=k+1$}
\ENDWHILE
\RETURN{$\{\bm v_n^k\}_{n=1}^{S}$}
\end{algorithmic}
\end{algorithm}

\section{Convergence Analysis}\label{sec:convergence}

In this section, we study convergence properties of the proposed Algorithm~\ref{alg:cs-lsm-pnp}. To facilitate the analysis, we recast the summation-based model \eqref{eq:model-axial} into an equivalent matrix-vector formulation and introduce two auxiliary functions that express the ADMM updates in a form amenable to convergence analysis. 

For clarity, we summarize the notation used throughout this section. Let $S:=NR$ be the total number of axial slices. The indices  $n=1,\dots,S$, $j=1,\dots,N$, and $r=1,\dots,R$ refer to individual axial slices, compressed shots, and local slices within a shot, respectively, with $n=(j-1)R+r$.  We stack the slice variables into full vectors
\[
\bm u := [\bm u_1^\top,\ldots,\bm u_S^\top]^\top \in \mathbb{R}^{pS},\qquad
\bm v := [\bm v_1^\top,\ldots,\bm v_S^\top]^\top \in \mathbb{R}^{pS},\qquad
\bm d := [\bm d_1^\top,\ldots,\bm d_S^\top]^\top \in \mathbb{R}^{pS},
\]
and write $\bm u^k$, $\bm v^k$, and $\bm d^k$ for these stacked vectors at the $k$th outer iteration. The $n$th slice variables at iteration $k$ are denoted 
$\bm u_n^k$, $\bm v_n^k$, and $\bm d_n^k$. 
For each shot $j$, the corresponding $R$ slice variables are collected into the shot-wise vector $\bm v^{(j)}$, and we write $\bm v^{(j),k}$ 
for its value at the $k$th iteration.
In the cyclic GS sweep, each slice variable $\bm u_n$ is treated as one coordinate block, which we refer to as a slice block; and $\bm u_{n,-}^k$ and $\bm u_{n,+}^k$ denote the previous and next neighboring slices used when updating $\bm u_n$, as defined in \eqref{eq:u-periodic}.
The notation $\bm U^{k,n}$ denotes the intermediate stacked vector after the first $n$ slice blocks have been updated in the $k$th GS sweep. Both $\bm v^{(j)}$ and $\bm U^{k,n}$ are defined explicitly in \eqref{eq:v^j} and \eqref{eq:U}, respectively.
If $\bm u^\star$ is an accumulation point,  $\bm u_n^\star$ denotes its $n$th slice. All notation is summarized in Table \ref{tab:notation_summary}.

\begin{table}[t]
\centering
\caption{Summary of notation used in the convergence analysis.}
\label{tab:notation_summary}
\begin{tabular}{ll}
\hline
Notation & Meaning \\
\hline
$n$
& Global axial-slice index, $n=1,\dots,S$ \\
$j$
& Compressed-shot index, $j=1,\dots,N$ \\
$r$
& Local slice index within the $j$th compressed shot, $r=1,\dots,R$ \\
$\bm u_n^k,\bm v_n^k,\bm d_n^k$
& The $n$th slice variable at the $k$th outer iteration \\
$\bm u^k,\bm v^k,\bm d^k$
& Stacked vectors at the $k$th outer iteration \\
$\bm v^{(j)}$
& Shot-wise vector collecting the $R$ slices associated with the $j$th shot \\
$\bm v^{(j),k}$
& Value of the shot-wise vector $\bm v^{(j)}$ at the $k$th iteration \\
$\bm u_{n,-}^k,\bm u_{n,+}^k$
& Previous and next neighboring slice used when updating the $n$th slice \\
$\bm U^{k,n}$
& Stacked vector after the first $n$ slices have been updated in the $k$th GS sweep \\
$\bm u_n^\star$
& The $n$th slice of an accumulation point $\bm u^\star$ \\
\hline
\end{tabular}
\end{table}

Recall that for each compressed shot $j=1,\dots,N$, we collect the corresponding $R$ slice variables into the shot-wise vector 
\begin{equation}\label{eq:v^j}
\bm v^{(j)} := [\bm v_{(j-1)R+1}^\top,\ldots,\bm v_{jR}^\top]^\top \in\mathbb R^{pR},
\end{equation}
and define the shot-wise measurement matrix
$
\Phi := [\bm\phi_1,\ldots,\bm\phi_R]\in\mathbb R^{p\times pR}.
$
Since the same set of masks is applied in every shot, the global forward operator takes the block-diagonal form
\[
A:=I_N\otimes \Phi\in\mathbb R^{pN\times pS},
\]
where $\otimes$ denotes the Kronecker product. Stacking all compressed measurements gives
\[
\bm b := [\bm b_1^\top,\ldots,\bm b_N^\top]^\top \in\mathbb R^{pN}.
\]
The summation in the data term of \eqref{eq:model-axial} can then be equivalently written in matrix-vector form as
\[
\frac12\|A\bm v-\bm b\|_2^2
=
\frac12\sum_{j=1}^{N}\left\|\bm b_j-\sum_{r=1}^{R}\bm\phi_r\bm v_{(j-1)R+r}\right\|_2^2.
\]

For the axial coupling term, define the periodic first-order difference operator 
\[
D:\mathbb R^{pS}\to\mathbb R^{pS},\qquad
(D\bm u)_n:=\bm u_n-\bm u_{n-1},\qquad \bm u_0:=\bm u_S.
\]
Then $\|D\bm u\|_2^2=\sum_{n=1}^{S}\|\bm u_n-\bm u_{n-1}\|_2^2.$
We adopt periodic boundary conditions here for analytical convenience; alternative nonperiodic boundary treatments can be handled similarly.

The slice-based and axial-coupled models can be written in the following unified formulation with $\gamma\ge 0$:
\begin{equation}\label{eq:convex_unified_stacked}
\min_{\bm v\in\mathbb{R}^{pS}}\;
F_{\gamma}(\bm v)
:=
\frac12\|A\bm v-\bm b\|_2^2
+\lambda\sum_{n=1}^{S}\psi(\bm v_n)
+\frac{\gamma}{2}\|D\bm v\|_2^2 .
\end{equation}
When $\gamma=0$, \eqref{eq:convex_unified_stacked} reduces to the slice-based model; when $\gamma>0$, it corresponds to the axial-coupled model. 

Introduce an auxiliary variable $\bm u$ and impose the constraint $\bm v=\bm u$. Then \eqref{eq:convex_unified_stacked} can be written as
\begin{equation}\label{eq:admm_form}
\min_{\bm v,\bm u}\; f(\bm v)+g(\bm u)
\quad\text{s.t.}\quad
\bm v-\bm u=\bm 0,
\end{equation}
where
\[
f(\bm v):=\frac12\|A\bm v-\bm b\|_2^2,
\qquad
g(\bm u):=
\lambda\sum_{n=1}^{S}\psi(\bm u_n)
+\frac{\gamma}{2}\|D\bm u\|_2^2 .
\]

Using the scaled dual variable $\bm d$, the augmented Lagrangian can be written as
\begin{equation}\label{eq:aug_lag_analysis}
\mathcal{L}_\rho(\bm u,\bm v,\bm d)
=
f(\bm v)+g(\bm u)
+\frac{\rho}{2}\|\bm v-\bm u+\bm d\|_2^2
-\frac{\rho}{2}\|\bm d\|_2^2 ,
\end{equation}
or equivalently,
\begin{equation}\label{eq:aug_lag_expanded}
\mathcal{L}_\rho(\bm u,\bm v,\bm d)
=
f(\bm v)+g(\bm u)
+\rho\langle \bm v-\bm u,\bm d\rangle
+\frac{\rho}{2}\|\bm v-\bm u\|_2^2 .
\end{equation}

Throughout this section, $\{(\bm u^k,\bm v^k,\bm d^k)\}_{k\ge0}$ denotes the sequence generated by Algorithm~\ref{alg:cs-lsm-pnp}. 

We now proceed with the convergence analysis. We first review relevant concepts from variational analysis \cite{rockafellar1998variational}, then establish a sufficient decrease property for the augmented Lagrangian in Theorem \ref{thm:sufficient_decrease}. Building on this and under a standard boundedness assumption, Corollary \ref{cor:asymptotic_regularity} establishes  asymptotic regularity of the iterates, and Theorem \ref{thm:cluster_optimality} shows that every accumulation point of the generated sequence is a stationary point of \eqref{eq:convex_unified_stacked}.

Recall that an extended real-valued
function $h:\mathbb{R}^m\to(-\infty,+\infty]$ is called proper if
$\operatorname{dom} h:=\{x\in\mathbb{R}^m: h(x)<+\infty\}$ is nonempty and
$h(x)>-\infty$ for all $x$. Additionally, a proper function $h$ is said to be closed if it is lower semi-continuous.
For a proper closed convex function $h$, its subdifferential at
$x\in\operatorname{dom}h$ is defined by 
$
\partial h(x)
    :=
    \{s\in\mathbb{R}^m:
    h(y)\ge h(x)+\langle s,y-x\rangle,\ \forall y\in\mathbb{R}^m\}.
$ When $h$ is differentiable, we write $\nabla h$ for its gradient, in which case $\partial h(x) =\{\nabla h(x)\}.$
A function $h:\mathbb R^m\to(-\infty,+\infty]$ is called
$\kappa$-weakly convex if
$
h(\cdot)+\frac{\kappa}{2}\|\cdot\|_2^2
$
is convex for a constant $\kappa\ge0$, and its 
 subdifferential is defined as
$
\partial h(x)
:=
\partial\left(h+\frac{\kappa}{2}\|\cdot\|_2^2\right)(x)-\kappa x.
$

\begin{assumption}\label{ass:psi}
$\psi:\mathbb R^p\to(-\infty,+\infty]$ is proper, closed, and $\kappa$-weakly convex. 
\end{assumption}

The Lipschitz constant of $\nabla f$ is given by 
$L_f:=\|A\|_2^2,$
which equals the largest eigenvalue of $A^\top A$. With these ingredients in place, the following theorem establishes a sufficient decrease property for the augmented Lagrangian along the iterates of  Algorithm \ref{alg:cs-lsm-pnp}.

\begin{theorem}\label{thm:sufficient_decrease}
Suppose Assumption~\ref{ass:psi} holds and the penalty parameter satisfies
$\rho>2L_f$ and $\rho+2\gamma>\kappa\lambda$. 
Then for all $k\ge 1$, 
\begin{align}
&\mathcal{L}_\rho(\bm u^{k+1},\bm v^{k+1},\bm d^{k+1})
-
\mathcal{L}_\rho(\bm u^k,\bm v^k,\bm d^k)
\notag\\
&\le
-\frac{\rho+2\gamma-\kappa\lambda}{2}\|\bm u^{k+1}-\bm u^k\|_2^2
-\frac{\rho}{2}\|\bm v^{k+1}-\bm v^k\|_2^2
-\left(\frac12-\frac{L_f}{\rho}\right)
\|A(\bm v^{k+1}-\bm v^k)\|_2^2 .
\label{eq:sufficient_decrease}
\end{align}
\end{theorem}

\begin{proof}
We estimate the changes of the augmented Lagrangian along the three updates.

Fixing $\bm v^k$ and $\bm d^k$, we begin with the $\bm u$-update. For the $n$th slice update $(n=1, \cdots, S)$, we define the single-slice objective
\begin{align}
Q_n^k(\bm z)
:=
\lambda\psi(\bm z)
+\frac{\rho}{2}\|\bm z-(\bm v_n^k+\bm d_n^k)\|_2^2
+\frac{\gamma}{2}\|\bm z-\bm u_{n,-}^k\|_2^2
+\frac{\gamma}{2}\|\bm z-\bm u_{n,+}^k\|_2^2 ,
\label{eq:block_objective}
\end{align}
where $\bm u_{n,-}^k$ and $\bm u_{n,+}^k$ are defined in \eqref{eq:u-periodic}.
To handle the GS sweep compactly, we introduce $\bm U^{k,n}$ to denote the intermediate iterate of the stacked vector after the first $n$ slices have been updated in the $k$th ADMM iteration, that is,
\begin{equation}\label{eq:U}
\bm U^{k,n}
:=
[(\bm u_1^{k+1})^\top,\ldots,(\bm u_n^{k+1})^\top,
(\bm u_{n+1}^{k})^\top,\ldots,(\bm u_S^{k})^\top]^\top,
\end{equation}
with $\bm U^{k,0}:=\bm u^k$ and $\bm U^{k,S}:=\bm u^{k+1}$.
At the time of the $n$th slice update, $Q_n^k$ contains exactly the terms in
$\mathcal{L}_\rho(\cdot,\bm v^k,\bm d^k)$ that depend on the slice $\bm u_n$, with all other slices fixed at their most recently computed values.
The remaining terms in $\mathcal{L}_\rho$ are independent of $\bm z$ and cancel in the difference, which implies
\[
\mathcal{L}_\rho(\bm U^{k,n},\bm v^k,\bm d^k)
-
\mathcal{L}_\rho(\bm U^{k,n-1},\bm v^k,\bm d^k)
=
Q_n^k(\bm u_n^{k+1})-Q_n^k(\bm u_n^k).
\]
Under Assumption~\ref{ass:psi} and the the condition that $\rho+2\gamma>\kappa\lambda$, $Q_n^k$ is $(\rho+2\gamma-\kappa\lambda)$-strongly convex. Since the $n$th slice update minimizes $Q_n^k$, we have
\[
Q_n^k(\bm u_n^{k+1})-Q_n^k(\bm u_n^k)
\le
-\frac{\rho+2\gamma-\kappa\lambda}{2}\|\bm u_n^{k+1}-\bm u_n^k\|_2^2 .
\]
Consequently,
\[
\mathcal{L}_\rho(\bm U^{k,n},\bm v^k,\bm d^k)
-
\mathcal{L}_\rho(\bm U^{k,n-1},\bm v^k,\bm d^k)
\le
-\frac{\rho+2\gamma-\kappa\lambda}{2}\|\bm u_n^{k+1}-\bm u_n^k\|_2^2 .
\]
Summing over $n=1,\dots,S$ yields
\begin{equation}\label{eq:u_decrease}
\mathcal{L}_\rho(\bm u^{k+1},\bm v^k,\bm d^k)
-
\mathcal{L}_\rho(\bm u^k,\bm v^k,\bm d^k)
\le
-\frac{\rho+2\gamma-\kappa\lambda}{2}\|\bm u^{k+1}-\bm u^k\|_2^2 .
\end{equation}

Next, we estimate the decrease from the $\bm v$-update. Let
\[
\Delta\bm v^{k+1}:=\bm v^{k+1}-\bm v^k .
\]
For fixed $\bm u^{k+1}$ and $\bm d^k$, define the $\bm v$-dependent quadratic function
\[
q_k(\bm v)
:=
\frac12\|A\bm v-\bm b\|_2^2
+
\rho\langle \bm v-\bm u^{k+1},\bm d^k\rangle
+
\frac{\rho}{2}\|\bm v-\bm u^{k+1}\|_2^2 .
\]
By the expanded form \eqref{eq:aug_lag_expanded}, the difference,
\[
\mathcal{L}_\rho(\bm u^{k+1},\bm v^{k+1},\bm d^k)
-
\mathcal{L}_\rho(\bm u^{k+1},\bm v^k,\bm d^k),
\]
is equal to $q_k(\bm v^{k+1})-q_k(\bm v^k)$. Since $q_k$ is quadratic with Hessian $A^\top A+\rho I$, we have
\begin{align}
q_k(\bm v^{k+1})-q_k(\bm v^k)
=&\;
\left\langle \nabla q_k(\bm v^{k+1}),\Delta\bm v^{k+1}\right\rangle
-\frac12\|A\Delta\bm v^{k+1}\|_2^2
-\frac{\rho}{2}\|\Delta\bm v^{k+1}\|_2^2 .
\label{eq:v_step_expand}
\end{align}
The optimality condition of the $\bm v$-update is
\[
\nabla q_k(\bm v^{k+1})
=
A^\top(A\bm v^{k+1}-\bm b)
+
\rho(\bm v^{k+1}-\bm u^{k+1}+\bm d^k)
=
\bm 0.
\]
Therefore,
\begin{equation}\label{eq:v_decrease}
\mathcal{L}_\rho(\bm u^{k+1},\bm v^{k+1},\bm d^k)
-
\mathcal{L}_\rho(\bm u^{k+1},\bm v^k,\bm d^k)
=
-\frac12\|A\Delta\bm v^{k+1}\|_2^2
-\frac{\rho}{2}\|\Delta\bm v^{k+1}\|_2^2 .
\end{equation}

It remains to control the change caused by the dual update. From the $\bm v$-optimality condition and
\[
\bm d^{k+1}=\bm d^k+\bm v^{k+1}-\bm u^{k+1},
\]
we obtain
\begin{equation}\label{eq:dual_relation}
A^\top(A\bm v^{k+1}-\bm b)+\rho\bm d^{k+1}=\bm 0,
\end{equation}
or equivalently,
\begin{equation}
\label{eq:dual_relation2}    \bm d^{k+1}
=
-\frac1\rho A^\top(A\bm v^{k+1}-\bm b).
\end{equation}
Applying this relation at two consecutive iterations gives, for $k\ge 1$,
\begin{equation}\label{eq:difference_dual}
\bm d^{k+1}-\bm d^k
=
-\frac1\rho A^\top A(\bm v^{k+1}-\bm v^k).
\end{equation}
Using the expanded form \eqref{eq:aug_lag_expanded}, we also have
\[
\mathcal{L}_\rho(\bm u^{k+1},\bm v^{k+1},\bm d^{k+1})
-
\mathcal{L}_\rho(\bm u^{k+1},\bm v^{k+1},\bm d^k)
=
\rho\left\langle
\bm v^{k+1}-\bm u^{k+1},
\bm d^{k+1}-\bm d^k
\right\rangle .
\]
It follows from $\bm v^{k+1}-\bm u^{k+1}=\bm d^{k+1}-\bm d^k$ together with \eqref{eq:difference_dual} that 
\begin{align}
 \notag   &\mathcal{L}_\rho(\bm u^{k+1},\bm v^{k+1},\bm d^{k+1})
-
\mathcal{L}_\rho(\bm u^{k+1},\bm v^{k+1},\bm d^k)\\
=&
\rho\|\bm d^{k+1}-\bm d^k\|_2^2 =
\frac1{\rho}
\|A^\top A(\bm v^{k+1}-\bm v^k)\|_2^2
\le
\frac{L_f}{\rho}
\|A(\bm v^{k+1}-\bm v^k)\|_2^2 . \label{eq:dual_increase}
\end{align}

Combining the estimates for the $\bm u$-step in \eqref{eq:u_decrease}, the $\bm v$-step in \eqref{eq:v_decrease}, and the dual update in \eqref{eq:dual_increase} gives \eqref{eq:sufficient_decrease}.
\end{proof}

\begin{assumption}\label{ass:psi_lower}
The function $\psi$ is bounded below.
\end{assumption}

This assumption holds for commonly used nonnegative regularizers, such as Tikhonov and TV regularization.

\begin{lemma}\label{lem:lower_bounded}
Suppose Assumption~\ref{ass:psi_lower} holds and $\rho>2L_f$. Then the augmented Lagrangian sequence $\{\mathcal{L}_\rho(\bm u^k,\bm v^k,\bm d^k)\}_{k\ge 1}$ is bounded below.
\end{lemma}

\begin{proof}
Using the relation at the $k$th iteration, we obtain
\begin{align}
   \notag \mathcal{L}_\rho(\bm u^k,\bm v^k,\bm d^k)
=&
f(\bm v^k)+g(\bm u^k)
+\frac{\rho}{2}\|\bm v^k-\bm u^k+\bm d^k\|_2^2
-\frac{\rho}{2}\|\bm d^k\|_2^2 \\
\ge&
f(\bm v^k)+g(\bm u^k)-\frac{1}{2\rho}\|A^\top(A\bm v^{k}-\bm b)\|_2^2.\label{ineq:lower_bound}
\end{align}
Since $f(\bm v)=\frac12\|A\bm v-\bm b\|_2^2$, it is straightforward that
\[
\|\nabla f(\bm v)\|_2^2
=
\|A^\top(A\bm v-\bm b)\|_2^2
\le
\|A\|_2^2\|A\bm v-\bm b\|_2^2
=
2L_f f(\bm v).
\]
We further plug the expression of $g(\bm u^k)$ into \eqref{ineq:lower_bound} and obtain
\[
\mathcal{L}_\rho(\bm u^k,\bm v^k,\bm d^k)
\ge
\left(1-\frac{L_f}{\rho}\right)f(\bm v^k)+g(\bm u^k) =\left(1-\frac{L_f}{\rho}\right)f(\bm v^k)+\lambda\sum_{n=1}^{S}\psi(\bm u_n^k)
+\frac{\gamma}{2}\|D\bm u^k\|_2^2.
\]
Since $\rho>2L_f$, we have $1-L_f/\rho>0$. The boundedness of $\psi$ in Assumption~\ref{ass:psi_lower} guarantees that $\{\mathcal{L}_\rho(\bm u^k,\bm v^k,\bm d^k)\}_{k\ge1}$ is bounded below.
\end{proof}

\begin{corollary}\label{cor:asymptotic_regularity}
Suppose Assumptions~\ref{ass:psi}--\ref{ass:psi_lower} hold, $\rho>2L_f,$ and $\rho+2\gamma>\kappa\lambda$. Then $\sum_{k=1}^{\infty}\|\bm u^{k+1}-\bm u^k\|_2^2<\infty$ and $\sum_{k=1}^{\infty}\|\bm v^{k+1}-\bm v^k\|_2^2<\infty$. Consequently,
\[
\|\bm u^{k+1}-\bm u^k\|_2\to 0,\qquad
\|\bm v^{k+1}-\bm v^k\|_2\to 0,\qquad
\|\bm d^{k+1}-\bm d^k\|_2\to 0,
\]
and the primal residual satisfies $\|\bm v^{k+1}-\bm u^{k+1}\|_2\to 0$. 
\end{corollary}

\begin{proof}
Since $\rho>2L_f$ and $\rho+2\gamma>\kappa\lambda$, the coefficients $\frac12-\frac{L_f}{\rho}$ and $\frac{\rho+2\gamma-\kappa\lambda}{2}$ in \eqref{eq:sufficient_decrease} are positive. Summing \eqref{eq:sufficient_decrease} over $k$ and using Lemma~\ref{lem:lower_bounded} gives $\sum_{k=1}^{\infty}\|\bm u^{k+1}-\bm u^k\|_2^2<\infty$ and $\sum_{k=1}^{\infty}\|\bm v^{k+1}-\bm v^k\|_2^2<\infty$. Therefore, $\|\bm u^{k+1}-\bm u^k\|_2\to 0$ and $\|\bm v^{k+1}-\bm v^k\|_2\to 0$. The relation \eqref{eq:difference_dual} implies $\|\bm d^{k+1}-\bm d^k\|_2\to 0$. Finally, by the dual update, $\bm v^{k+1}-\bm u^{k+1}=\bm d^{k+1}-\bm d^k$, and hence the primal residual also converges to zero.
\end{proof}

\begin{lemma}\label{lem:relative_error}
Suppose Assumption~\ref{ass:psi} holds. For every $k\ge 1$, there exist vectors
$\bm e_u^{k+1}$, $\bm e_v^{k+1}$, and $\bm e_d^{k+1}$ such that
\begin{align}
\bm e_u^{k+1}
&\in
\partial g(\bm u^{k+1})
-\rho(\bm v^{k+1}-\bm u^{k+1}+\bm d^{k+1}),
\label{eq:relative_error_u}\\
\bm e_v^{k+1}
&=
\nabla f(\bm v^{k+1})
+\rho(\bm v^{k+1}-\bm u^{k+1}+\bm d^{k+1}),
\label{eq:relative_error_v}\\
\bm e_d^{k+1}
&=
\rho(\bm v^{k+1}-\bm u^{k+1}),
\label{eq:relative_error_d}
\end{align}
and
\begin{equation}\label{eq:relative_error_bound}
\|\bm e_u^{k+1}\|_2+\|\bm e_v^{k+1}\|_2+\|\bm e_d^{k+1}\|_2
\le
C_{\rm rel}
\left(
\|\bm u^{k+1}-\bm u^k\|_2
+
\|\bm v^{k+1}-\bm v^k\|_2
\right),
\end{equation}
where $C_{\rm rel}>0$ is independent of $k$. Consequently, under the assumptions of Corollary~\ref{cor:asymptotic_regularity},
\[
\|\bm e_u^{k+1}\|_2+\|\bm e_v^{k+1}\|_2+\|\bm e_d^{k+1}\|_2\to 0 .
\]
\end{lemma}

\begin{proof}
Let $\Delta\bm u^{k+1}:=\bm u^{k+1}-\bm u^k$, $\Delta\bm v^{k+1}:=\bm v^{k+1}-\bm v^k$, $\Delta\bm d^{k+1}:=\bm d^{k+1}-\bm d^k$.
The optimality condition of the $\bm v$-update gives
$
\nabla f(\bm v^{k+1})
+
\rho(\bm v^{k+1}-\bm u^{k+1}+\bm d^k)
=
\bm 0 .
$
Thus, with $\bm e_v^{k+1}$ defined by \eqref{eq:relative_error_v},
\begin{equation}\label{eq:ev_bound}
\|\bm e_v^{k+1}\|_2
=
\rho\|\Delta\bm d^{k+1}\|_2 .
\end{equation}
Moreover, the dual update gives
$
\bm v^{k+1}-\bm u^{k+1}=\Delta\bm d^{k+1}.
$
Thus, with $\bm e_d^{k+1}$ defined by \eqref{eq:relative_error_d},
\begin{equation}\label{eq:ed_bound}
\|\bm e_d^{k+1}\|_2
=
\rho\|\Delta\bm d^{k+1}\|_2 .
\end{equation}

It remains to estimate the $\bm u$-component. The optimality condition of \eqref{eq:axial_u_sub} implies that for every slice $n$, there exists
$\bm s_n^{k+1}\in\partial\psi(\bm u_n^{k+1})$ such that
\begin{equation}\label{eq:u_block_opt_relative}
\lambda\bm s_n^{k+1}
+
\rho(\bm u_n^{k+1}-\bm v_n^k-\bm d_n^k)
+
\gamma(\bm u_n^{k+1}-\bm u_{n,-}^{k})
+
\gamma(\bm u_n^{k+1}-\bm u_{n,+}^{k})
=
\bm 0 .
\end{equation}
Define
\[
\bm q_n^{k+1}
:=
\lambda\bm s_n^{k+1}
+
\gamma(2\bm u_n^{k+1}-\bm u_{n-1}^{k+1}-\bm u_{n+1}^{k+1}),
\qquad n=1,\dots,S,
\]
where $\bm u_0^{k+1}=\bm u_S^{k+1}$ and $\bm u_{S+1}^{k+1}=\bm u_1^{k+1}$, 
and let
$
\bm q^{k+1}:=
[(\bm q_1^{k+1})^\top,\ldots,(\bm q_S^{k+1})^\top]^\top .
$
By the definition of $g$, we have
$
\bm q^{k+1}\in\partial g(\bm u^{k+1}).
$
Now define
\begin{equation}\label{eq:eu_def_relative}
\bm e_u^{k+1}
:=
\bm q^{k+1}
-
\rho(\bm v^{k+1}-\bm u^{k+1}+\bm d^{k+1}).
\end{equation}
Then \eqref{eq:relative_error_u} holds. Moreover, by \eqref{eq:u_block_opt_relative}, the $n$th slice of $\bm e_u^{k+1}$ satisfies
\begin{equation}\label{eq:eu_identity_relative}
\bm e_{u,n}^{k+1}
=
-\rho(\bm v_n^{k+1}-\bm v_n^k)
-\rho(\bm d_n^{k+1}-\bm d_n^k)
+\gamma(\bm u_{n,-}^{k}-\bm u_{n-1}^{k+1})
+\gamma(\bm u_{n,+}^{k}-\bm u_{n+1}^{k+1}).
\end{equation}
The last two terms come from the mixed old and new neighboring slices in one cyclic Gauss--Seidel sweep. By \eqref{eq:u-periodic},
\[
\left(
\sum_{n=1}^{S}
\left\|
(\bm u_{n,-}^{k}-\bm u_{n-1}^{k+1})
+
(\bm u_{n,+}^{k}-\bm u_{n+1}^{k+1})
\right\|_2^2
\right)^{1/2}
\le
2\|\Delta\bm u^{k+1}\|_2 .
\]
Therefore,
\begin{equation}\label{eq:eu_bound}
\|\bm e_u^{k+1}\|_2
\le
2\gamma\|\Delta\bm u^{k+1}\|_2
+
\rho\|\Delta\bm v^{k+1}\|_2
+
\rho\|\Delta\bm d^{k+1}\|_2 .
\end{equation}

Finally, by \eqref{eq:difference_dual},
$
\Delta\bm d^{k+1}
=
-\frac1\rho A^\top A\Delta\bm v^{k+1},
$
and hence
$
\rho\|\Delta\bm d^{k+1}\|_2
\le
L_f\|\Delta\bm v^{k+1}\|_2 .
$
Combining this estimate with \eqref{eq:ev_bound}, \eqref{eq:ed_bound}, and \eqref{eq:eu_bound} gives \eqref{eq:relative_error_bound} with any
$
C_{\rm rel}\ge 2\gamma+\rho+3L_f .
$
The final claim follows from Corollary~\ref{cor:asymptotic_regularity}.
\end{proof}

\begin{theorem}\label{thm:cluster_optimality}
Suppose Assumptions~\ref{ass:psi}--\ref{ass:psi_lower} hold, $\rho>2L_f$, and $\rho+2\gamma>\kappa\lambda$. If the sequence $\{(\bm u^k,\bm v^k,\bm d^k)\}$ generated by Algorithm \ref{alg:cs-lsm-pnp} is bounded, then every accumulation point $(\bm u^\star,\bm v^\star,\bm d^\star)$ satisfies $\bm v^\star=\bm u^\star$ and
\begin{equation}\label{eq:kkt_limit}
\bm 0\in \nabla f(\bm v^\star)+\partial g(\bm v^\star).
\end{equation}
Consequently, $\bm v^\star$ is a stationary point of \eqref{eq:convex_unified_stacked}.
\end{theorem}

\begin{proof}
Let $(\bm u^\star,\bm v^\star,\bm d^\star)$ be an accumulation point of $\{(\bm u^k,\bm v^k,\bm d^k)\}$, which exists by boundedness and the Bolzano-Weierstrass theorem. Take a subsequence converging to it; without loss of generality, we index this subsequence by $k$ for simplicity. By Corollary~\ref{cor:asymptotic_regularity}, the shifted subsequence $(\bm u^{k+1},\bm v^{k+1},\bm d^{k+1})$ converges to the same limit. Since $\bm v^{k+1}-\bm u^{k+1}\to \bm 0$, we have $\bm v^\star=\bm u^\star$.

The optimality condition of the $\bm v$-update gives $\nabla f(\bm v^{k+1})+\rho\bm d^{k+1}=\bm 0.$
Passing to the limit yields
\begin{equation}\label{eq:v_limit_opt}
\nabla f(\bm v^\star)+\rho\bm d^\star=\bm 0.
\end{equation}
By Lemma~\ref{lem:relative_error}, there exist vectors
$\bm e_u^{k+1}$, $\bm e_v^{k+1}$, and $\bm e_d^{k+1}$ satisfying
\eqref{eq:relative_error_u}--\eqref{eq:relative_error_d}, such that
$
\|\bm e_u^{k+1}\|_2+\|\bm e_v^{k+1}\|_2+\|\bm e_d^{k+1}\|_2\to 0.
$
From \eqref{eq:relative_error_u}, there exists
$\bm q^{k+1}\in\partial g(\bm u^{k+1})$ such that
$
\bm e_u^{k+1}
=
\bm q^{k+1}
-
\rho(\bm v^{k+1}-\bm u^{k+1}+\bm d^{k+1}).
$
Since $\bm e_u^{k+1}\to\bm 0$, $\bm v^{k+1}-\bm u^{k+1}\to\bm 0$, and
$\bm d^{k+1}\to\bm d^\star$, we have
$
\bm q^{k+1}\to \rho\bm d^\star .
$
Under Assumption~\ref{ass:psi}, the function $g$ is proper, closed, and weakly convex, so its subdifferential is closed under limits. Since
$\bm q^{k+1}\in\partial g(\bm u^{k+1})$,
$\bm u^{k+1}\to\bm u^\star$, and
$\bm q^{k+1}\to\rho\bm d^\star$, we obtain
\begin{equation}\label{eq:u_limit_opt}
\rho\bm d^\star\in\partial g(\bm u^\star).
\end{equation}

Since $\bm u^\star=\bm v^\star$, combining \eqref{eq:v_limit_opt} and \eqref{eq:u_limit_opt} gives
\[
\bm 0\in \nabla f(\bm v^\star)+\partial g(\bm v^\star).
\]
Therefore, $\bm v^\star$ is a stationary point of \eqref{eq:convex_unified_stacked}.
\end{proof}

\begin{remark}
    The boundedness of the generated sequence $\{\bm u^k,\bm v^k,\bm d^k\}$ is assumed rather than proved. A sufficient condition that guarantees boundedness is coercivity \cite{wang2019global} of $\psi,$ which ensures that the augmented Lagrangian sublevel sets are compact. Both Tikhonov and TV regularizations satisfy this condition. Since coercivity may be difficult to verify for a general function $\psi,$ we state  boundedness directly as a hypothesis, which can be monitored empirically during the iteration.
\end{remark}

\begin{remark}
The weakly convex assumption on $\psi$ provides a natural middle ground  between convex and fully nonconvex regularization. Classical convergence guarantees \cite{boyd2011distributed,eckstein2015understanding} require convexity, which can be restrictive in practice, as some useful regularization models in image reconstruction are nonconvex. Weak convexity retains enough structure to support the convergence analysis above while accommodating a broader class of regularizers.  This connects to recent efforts to extend PnP convergence theory to nonconvex settings: for example, Hurault et al.~\cite{hurault2022proximal} propose a proximal denoising step that, in certain cases, coincides with the exact proximal operator of a possibly nonconvex regularizer, and Shoushtari et al.~\cite{pmlr-v235-shoushtari24a}  establish nonconvex convergence guarantees for PnP-ADMM under distribution shift, focusing on the case where a learned denoiser is applied outside its training distribution.
\end{remark}

\begin{remark}
The convergence analysis in this section is variational in nature: it assumes that the denoising step is induced by a regularization functional $\psi$. Another line of PnP research establishes convergence by imposing structural assumptions directly on the denoiser. For example, convergence results for PnP methods with black-box denoisers may require denoiser-based assumptions, such as nonexpansiveness, contractiveness, or averagedness of the denoiser itself or of a related residual operator \cite{chan2017plug,ryu2019plug}. Such results often lead to fixed-point convergence guarantees rather than convergence for a variational model. In contrast, our analysis assumes that $\psi$ is weakly convex and proves subsequential convergence for the corresponding variational formulation, and the result applies directly to Tikhonov and TV. For off-the-shelf denoisers such as BM3D, DnCNN, FFDNet, or DRUNet, the iteration is generally no longer equivalent to minimizing a fixed objective. Establishing or verifying such denoiser-based assumptions within the present framework is beyond the scope of this work; for these denoisers, we report empirical behavior only. Developing a rigorous convergence theory for black-box denoisers within the proposed axial-coupled PnP scheme remains an open problem for future work.
\end{remark}

\section{Experiments}\label{sec:experiments}

In this section, we evaluate the proposed reconstruction framework on synthetic and real zebrafish-heart CS-LSM data. 
The synthetic experiment provides a controlled setting with a known reference volume, whereas the real-data experiment evaluates the method under realistic acquisition conditions using a physical imaging system. The experiments evaluate the effects of denoiser choice, axial coupling, and the compression ratio. All experiments were implemented in Python and run on an Azure Machine Learning Standard\_E4ds\_v4 CPU instance with 4 cores and 32 GB RAM.

\subsection{Experimental setup}

\paragraph{Datasets}

For the synthetic experiment, we use a zebrafish-heart reference volume derived from LSM data acquired by our in-house LSM system~\cite{zhang2024_4d}. The underlying dataset was obtained using retrospective synchronization, which aligns image sequences from different axial slices and cardiac cycles. 
The selected reference volume ($200\times200\times120$-voxel) corresponds to a representative time point, containing both the atrium and ventricle with approximately $300$ nuclei near the chamber surfaces. Compressed measurements are generated from this reference using the forward model in \eqref{eq:forward_model} with known random binary masks, consistent with the DMD coding strategy in CS-LSM. The compression ratio is set to $R=5$ unless otherwise stated. Results for varying compression ratios are reported in Section~\ref{sect:exp4CR}.
No additional measurement noise is added in the synthetic experiments, so that the comparison isolates the effects of the inverse solver, the denoising prior, and the axial coupling term. 

For the real-data experiment, we use CS-LSM measurements and the corresponding sensing masks acquired by the DMD-based light-sheet platform described in \cite{zhang2026compressive}. The real dataset consists of $N=10$ compressed camera shots, each of size $404\times404$ pixels, acquired with compression ratio $R=15$, thus the reconstructed volume contains $S=150$ axial slices. These measurements are collected under imaging conditions comparable to those of the synthetic data in zebrafish-heart, but a corresponding volumetric ground truth is not available. To focus on reconstruction performance, we compare different models and denoisers directly on the acquired measurements without additional pre- or post-processing. While addressing effects such as rolling-shutter distortion may further improve image quality, we refer readers to our related work~\cite{zhang2026compressive} for a detailed discussion of these issues.

\paragraph{Evaluation metrics}

To assess reconstruction performance on synthetic data, we adopt two standard image-quality metrics: Peak Signal-to-Noise Ratio (PSNR) \cite{hore2010psnr} and Structural Similarity Index Measure (SSIM) \cite{wang2004ssim}. PSNR measures pixel-wise reconstruction fidelity, with larger values indicating smaller error. SSIM measures structural similarity between each reconstructed slice and the corresponding reference, following \cite{wang2004ssim}, with values closer to 1 indicating greater similarity. Both metrics are computed per slice and averaged over all valid slices. For real data, we report visual comparisons only.

\paragraph{Parameter tuning}

The PnP-ADMM framework has three hyperparameters: the ADMM penalty parameter $\rho$, the parameter $\lambda$ controlling the strength of the regularization or denoising step, and the axial coupling weight $\gamma$, with $\gamma=0$ for the slice-based model. For each denoiser and reconstruction model, we tune these parameters on the synthetic dataset using Bayesian optimization \cite{snoek2012practical,jones1998efficient}, with negative slice-averaged PSNR as the objective. Minimizing this objective is equivalent to maximizing the average reconstruction quality over the axial stack. Each Bayesian-optimization run is limited to $50$ objective evaluations over predefined search ranges, and the optimized parameters are then used to rerun the corresponding reconstruction. The optimized hyperparameters for all denoiser and model combinations are reported in Table \ref{tab:best_params_main}. For the real data, no separate parameter search is performed, as an exactly matched volumetric ground truth is unavailable; instead, parameters selected on the synthetic dataset are reused, given similar imaging conditions.

\begin{remark}
    The optimal parameter $\rho$ obtained via Bayesian optimization is notably smaller than the value required by Theorem~\ref{thm:cluster_optimality}. This is not a contradiction, as Theorem~\ref{thm:cluster_optimality} provides a sufficient (not necessary) condition for convergence of the axial-coupled PnP-ADMM scheme, and the true convergence region may be considerably larger. As shown in Table~\ref{tab:best_params_main}, smaller values of $\rho$ work well empirically and tend to yield better reconstruction quality, suggesting that the theoretical bound is conservative. Such a gap between sufficient conditions and practical performance is common in convergence analysis.
\end{remark}

\paragraph{Stopping criteria}

For synthetic data, the PnP-ADMM iterations are terminated when the relative change in iterates, i.e., $\|\bm v^{k}-\bm v^{k-1}\|_2/\|\bm v^{k-1}\|_2$,
falls below $10^{-3}$, or when the number of iterations reaches $100$. For real data, a looser tolerance of $10^{-2}$ is used to account for measurement noise. 

\begin{table}[t]
\centering
\caption{Hyperparameters selected by Bayesian optimization on the synthetic dataset.}
\vspace{0.3em}
\setlength{\tabcolsep}{5pt}
\begin{tabular}{lccccc}
\hline
\multirow{2}{*}{Denoiser}
& \multicolumn{2}{c}{Slice-based}
& \multicolumn{3}{c}{Axial-coupled} \\
\cline{2-3}\cline{4-6}
& $\rho$ & $\lambda$
& $\rho$ & $\lambda$ & $\gamma$ \\
\hline
Tikhonov & 0.0501 & 0.0198 & 0.0412 & 0.0248 & 0.0097 \\
TV       & 0.4999 & 0.0673 & 0.0068 & 0.0018 & 0.0010 \\
BM3D     & 0.7828 & 0.0075 & 0.0083 & 0.0001 & 0.0011 \\
DnCNN    & 0.0812 & 0.0008 & 0.0156 & 0.0002 & 0.0031 \\
FFDNet   & 0.1439 & 0.0014 & 0.0403 & 0.0004 & 0.0010 \\
DRUNet   & 0.1921 & 0.0018 & 0.0236 & 0.0002 & 0.0010 \\
\hline
\end{tabular}
\label{tab:best_params_main}
\end{table}

\subsection{Synthetic data}

\paragraph{Quantitative results}

\begin{table}[t]
\centering
\caption{Quantitative comparison on the synthetic dataset.}
\vspace{0.3em}
\setlength{\tabcolsep}{2.5pt}
\begin{tabular}{lcccccc}
\hline
\multirow{2}{*}{Denoiser}
& \multicolumn{2}{c}{PSNR (dB) $\uparrow$}
& \multicolumn{2}{c}{SSIM $\uparrow$}
& \multicolumn{2}{c}{Runtime (s) $\downarrow$} \\
\cline{2-3}\cline{4-5}\cline{6-7}
& Slice & Axial
& Slice & Axial
& Slice & Axial \\
\hline
Tikhonov & 29.7675 & 33.9182 & 0.7230 & 0.9283 & 0.579   & 1.267 \\
TV       & 36.4616 & 40.1035 & 0.9551 & 0.9770 & 1.861   & 3.674 \\
BM3D     & 37.8380 & 40.2827 & 0.9617 & 0.9718 & 392.380 & 501.994 \\
DnCNN    & 42.0069 & 42.8288 & 0.9751 & 0.9735 & 121.921 & 142.108 \\
FFDNet   & 43.1213 & 43.3507 & 0.9859 & 0.9866 & 61.005  & 68.156 \\
DRUNet   & 43.2100 & 43.7359 & 0.9866 & 0.9877 & 406.829 & 504.127 \\
\hline
\end{tabular}
\label{tab:synthetic_quantitative_main}
\end{table}

Table~\ref{tab:synthetic_quantitative_main} summarizes the results on synthetic data for the slice-based and axial-coupled models. Compared with slice-based reconstructions, axial coupling improves PSNR across all considered denoisers. SSIM likewise improves in most cases, with the only exception being DnCNN, where the difference is negligible.  Notably, the improvement is more pronounced for classical priors such as Tikhonov and TV, where slice-based reconstructions are less accurate, but smaller for deep-learning-based denoisers such as FFDNet and DRUNet, whose slice-based reconstructions are already close to the reference.

The gains from axial coupling reflect the multiplexed structure of the CS-LSM inverse problem. Under compressed acquisition, each camera measurement is formed by the masked superposition of $R$ axial slices. Although the data-fidelity term constrains these slices jointly, the slice-based model does not use neighboring slices as an additional prior. As a result, weak cellular or nuclear signals can be less stably localized along the axial direction and may appear fragmented or inconsistent across neighboring slices. The axial term therefore acts as a weak consistency constraint along $z$, encouraging neighboring slices to share structural information while avoiding strong axial smoothing. Except for Tikhonov, the selected $\gamma$ values in Table~\ref{tab:best_params_main} are on the order of $10^{-3}$, suggesting that weak axial coupling is effective for most denoisers.

The results reveal a clear dependence on the choice of denoiser. Tikhonov is the fastest method but yields the lowest PSNR and SSIM, while TV provides a stronger classical baseline with only a modest increase in runtime. BM3D improves upon TV in the slice-based setting, although its runtime is substantially higher due to the nonlocal patch search.  Among the DL-based denoisers, FFDNet and DRUNet achieve the best reconstruction quality. DRUNet attains the highest PSNR/SSIM in both the slice-based and axial-coupled models, whereas FFDNet achieves comparable accuracy with a much lower computational cost,  offering a better overall balance between reconstruction quality and efficiency.

\begin{figure*}
\centering
\setlength{\tabcolsep}{1pt}
\renewcommand{\arraystretch}{1.02}

\resizebox{0.8\textwidth}{!}{%
\begin{tabular}{@{}c c c c@{}}
\textbf{\footnotesize Slice-based slice 15} &
\textbf{\footnotesize Axial-coupled slice 15} &
\textbf{\footnotesize Slice-based MIP} &
\textbf{\footnotesize Axial-coupled MIP} \\[4pt]

\makebox[0pt][r]{\scriptsize\textbf{Ground truth}\hspace{0.6em}}%
\includegraphics[width=0.22\textwidth]{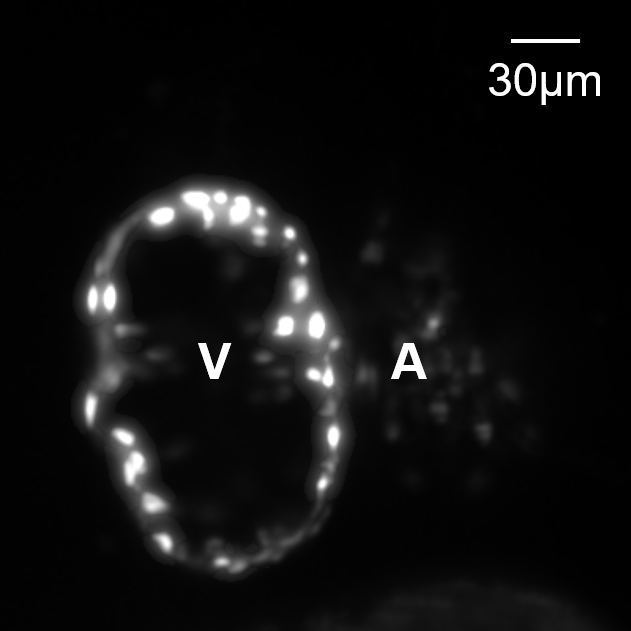} &
\includegraphics[width=0.22\textwidth]{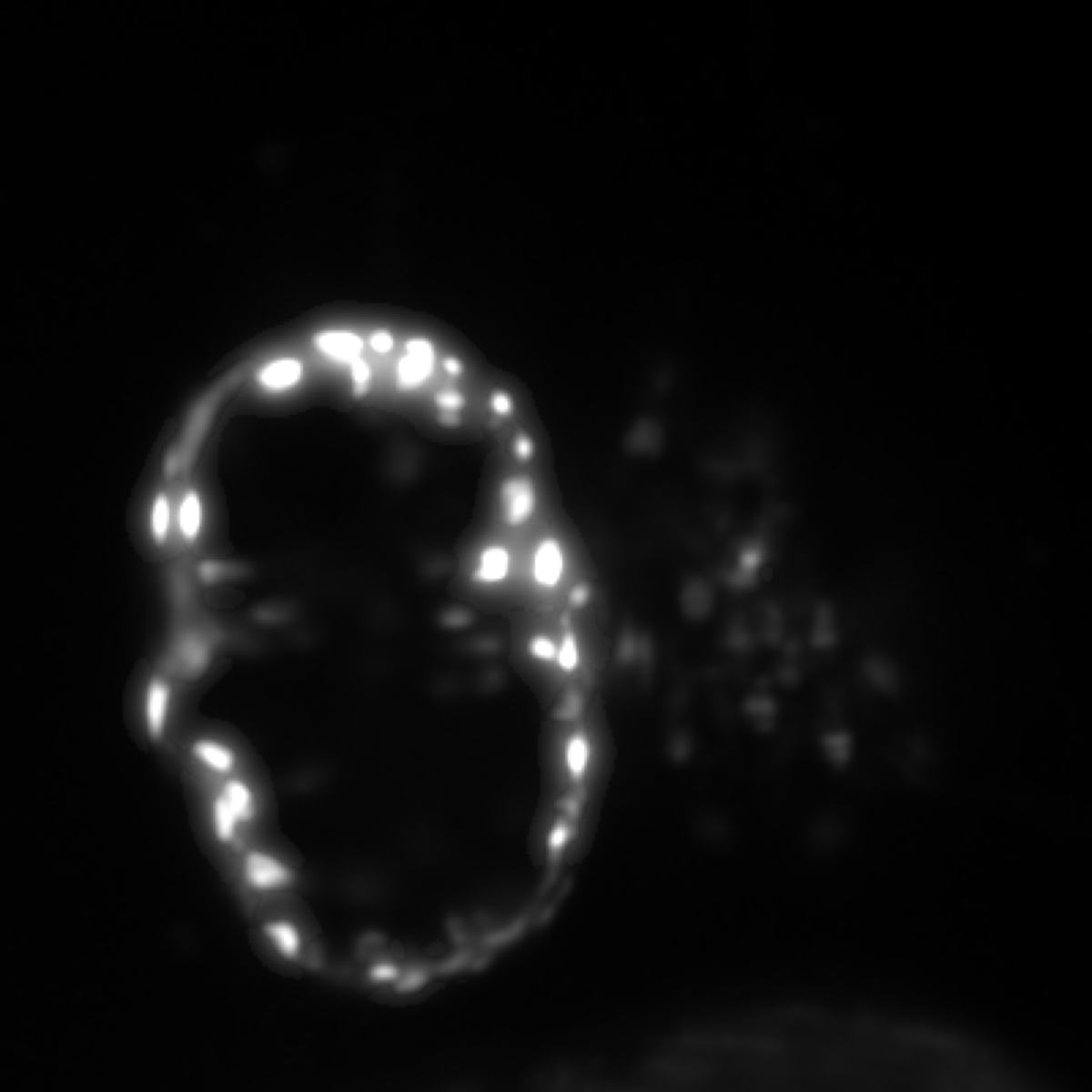} &
\includegraphics[width=0.22\textwidth]{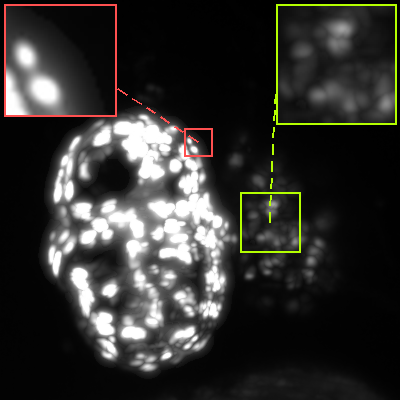} &
\includegraphics[width=0.22\textwidth]{Figures/syn_gt_mip-zoomin.png} \\[3pt]

\makebox[0pt][r]{\scriptsize\textbf{Tikhonov}\hspace{0.6em}}%
\includegraphics[width=0.22\textwidth]{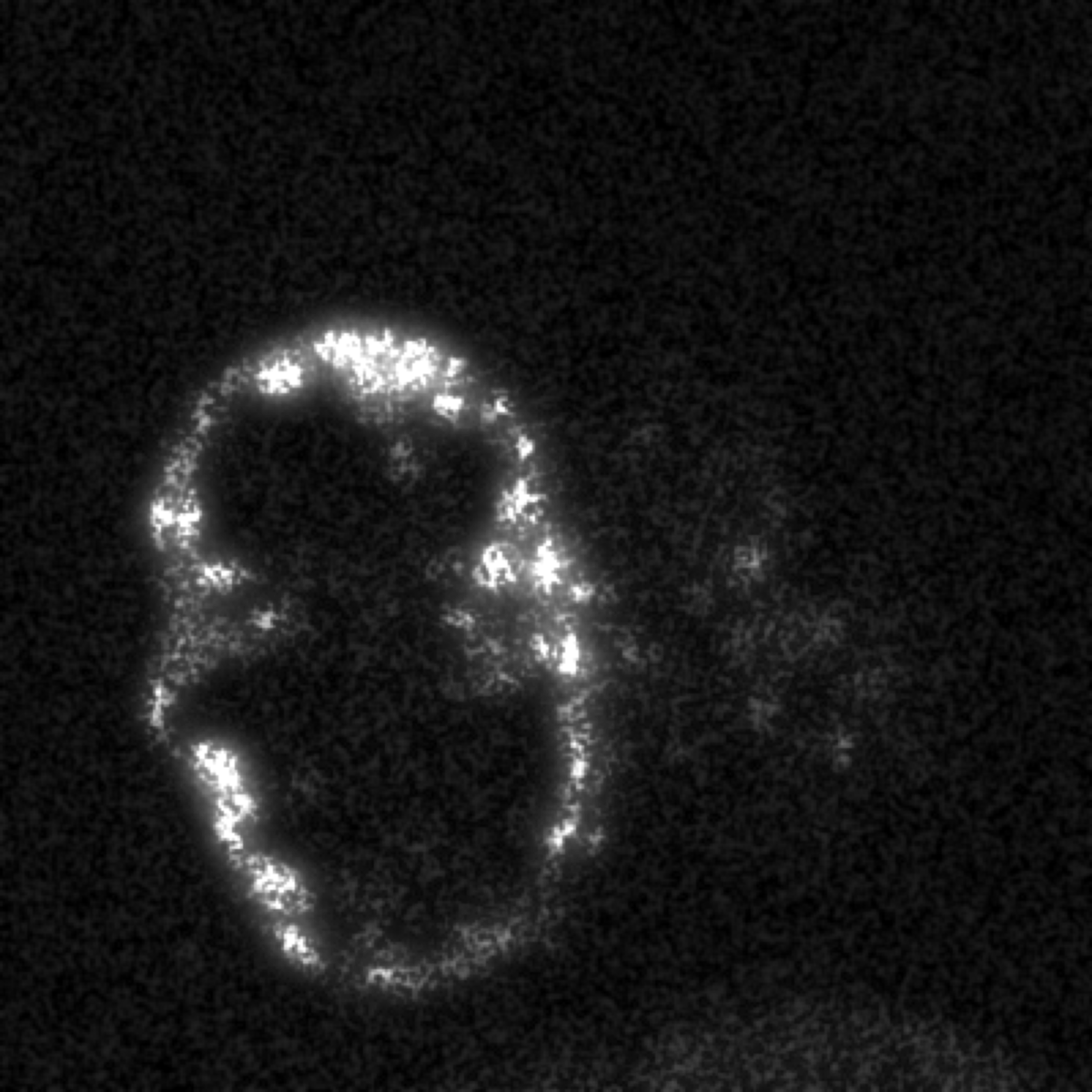} &
\includegraphics[width=0.22\textwidth]{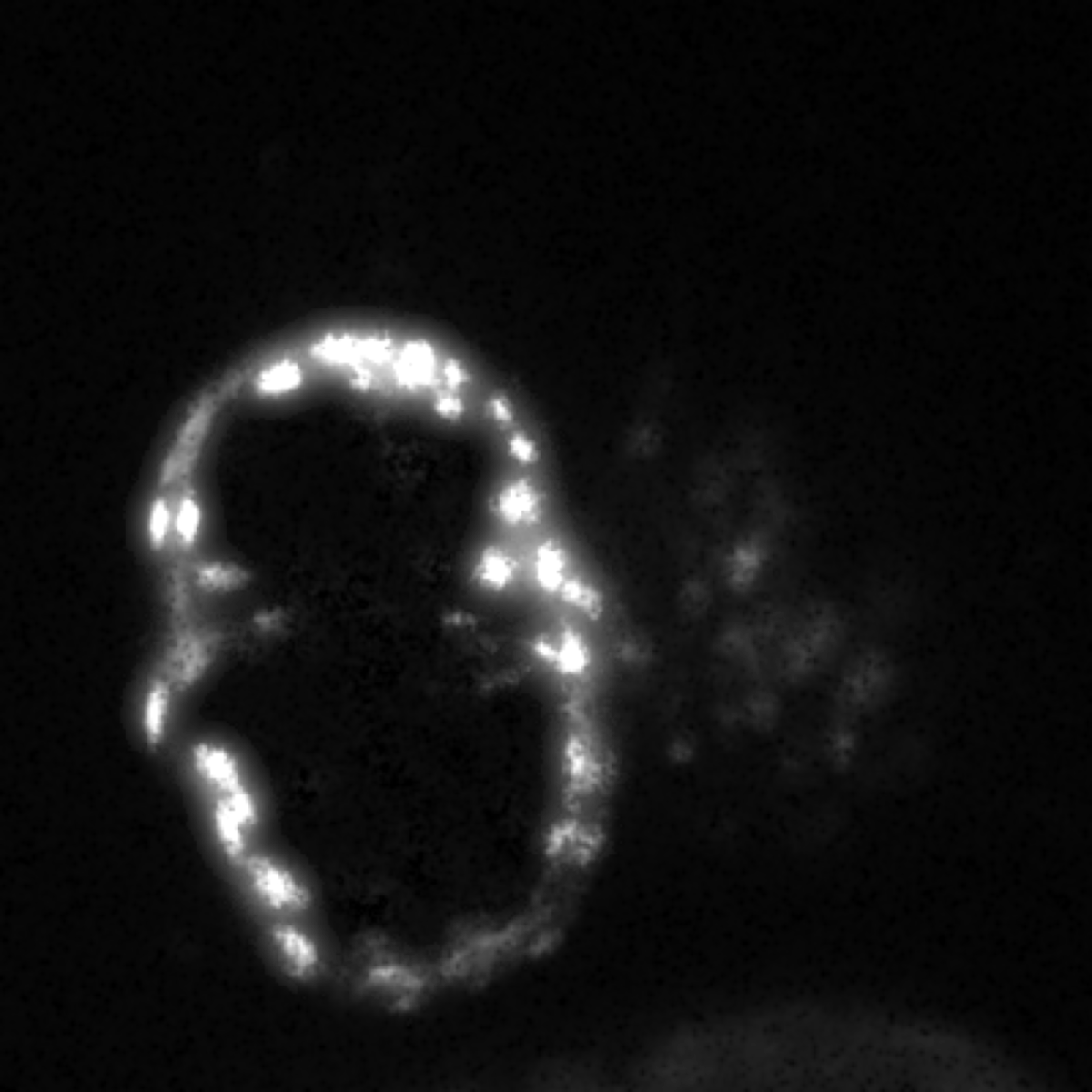} &
\includegraphics[width=0.22\textwidth]{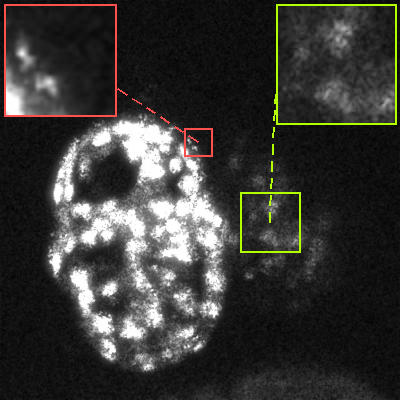} &
\includegraphics[width=0.22\textwidth]{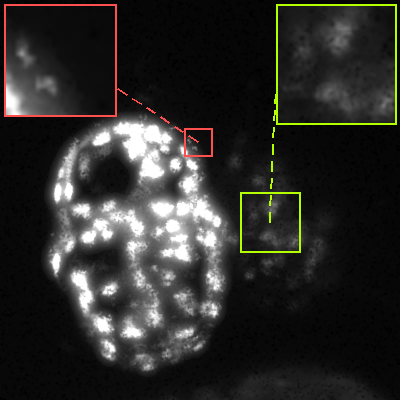} \\[3pt]

\makebox[0pt][r]{\scriptsize\textbf{TV}\hspace{0.6em}}%
\includegraphics[width=0.22\textwidth]{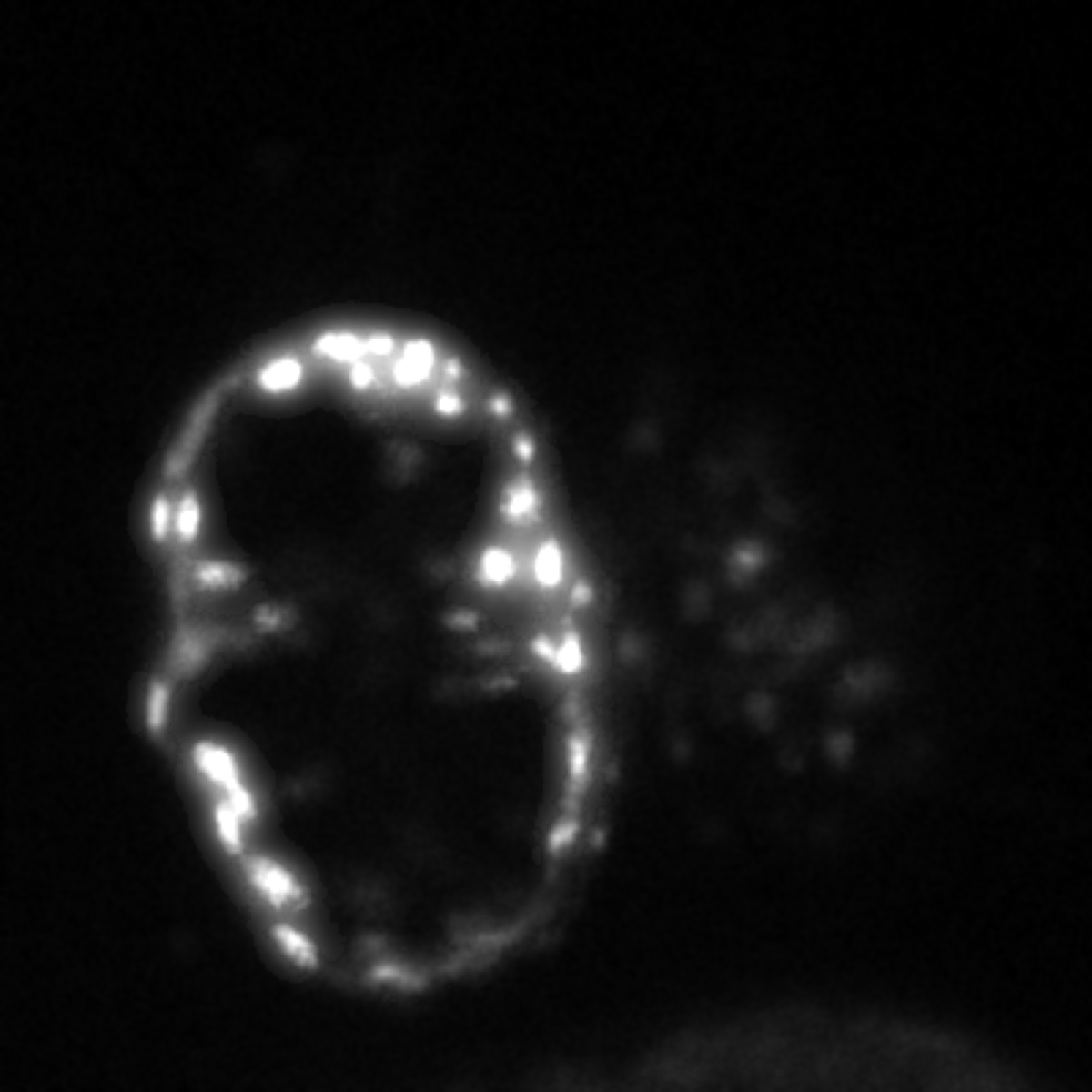} &
\includegraphics[width=0.22\textwidth]{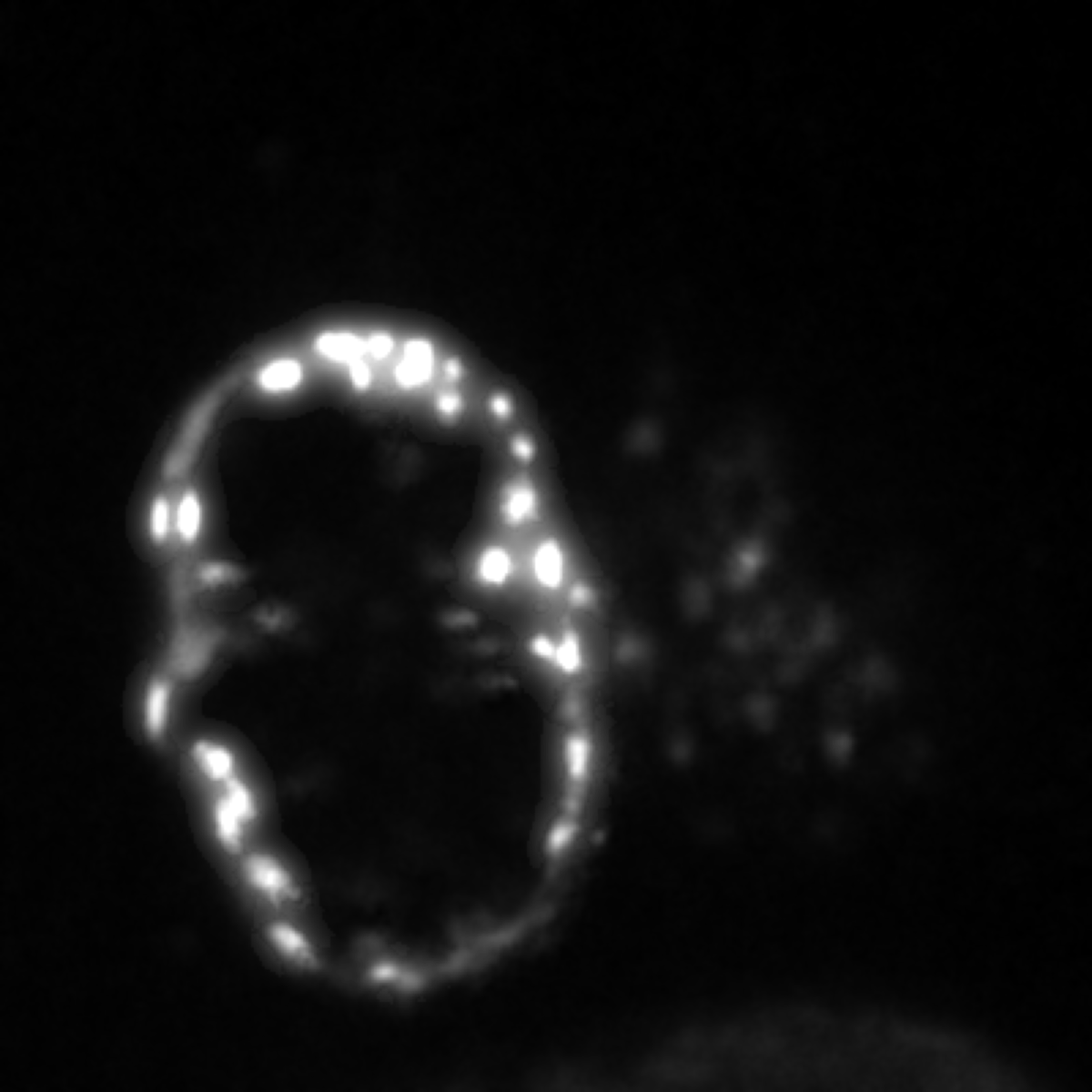} &
\includegraphics[width=0.22\textwidth]{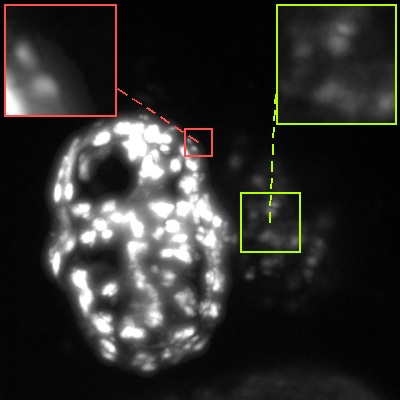} &
\includegraphics[width=0.22\textwidth]{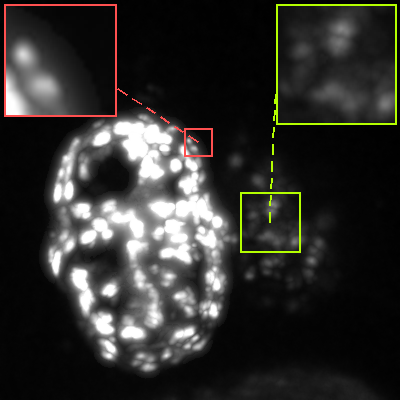} \\[3pt]

\makebox[0pt][r]{\scriptsize\textbf{BM3D}\hspace{0.6em}}%
\includegraphics[width=0.22\textwidth]{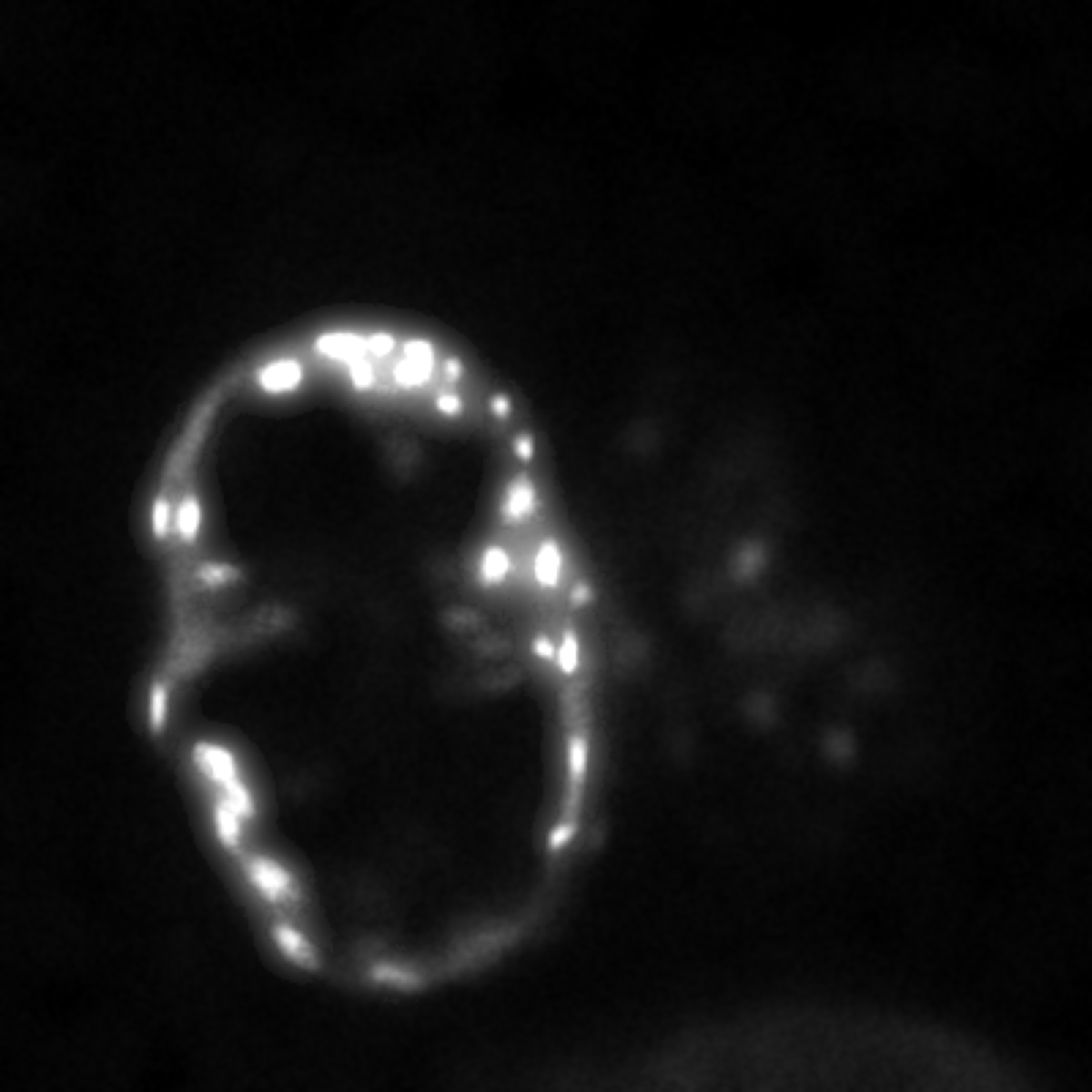} &
\includegraphics[width=0.22\textwidth]{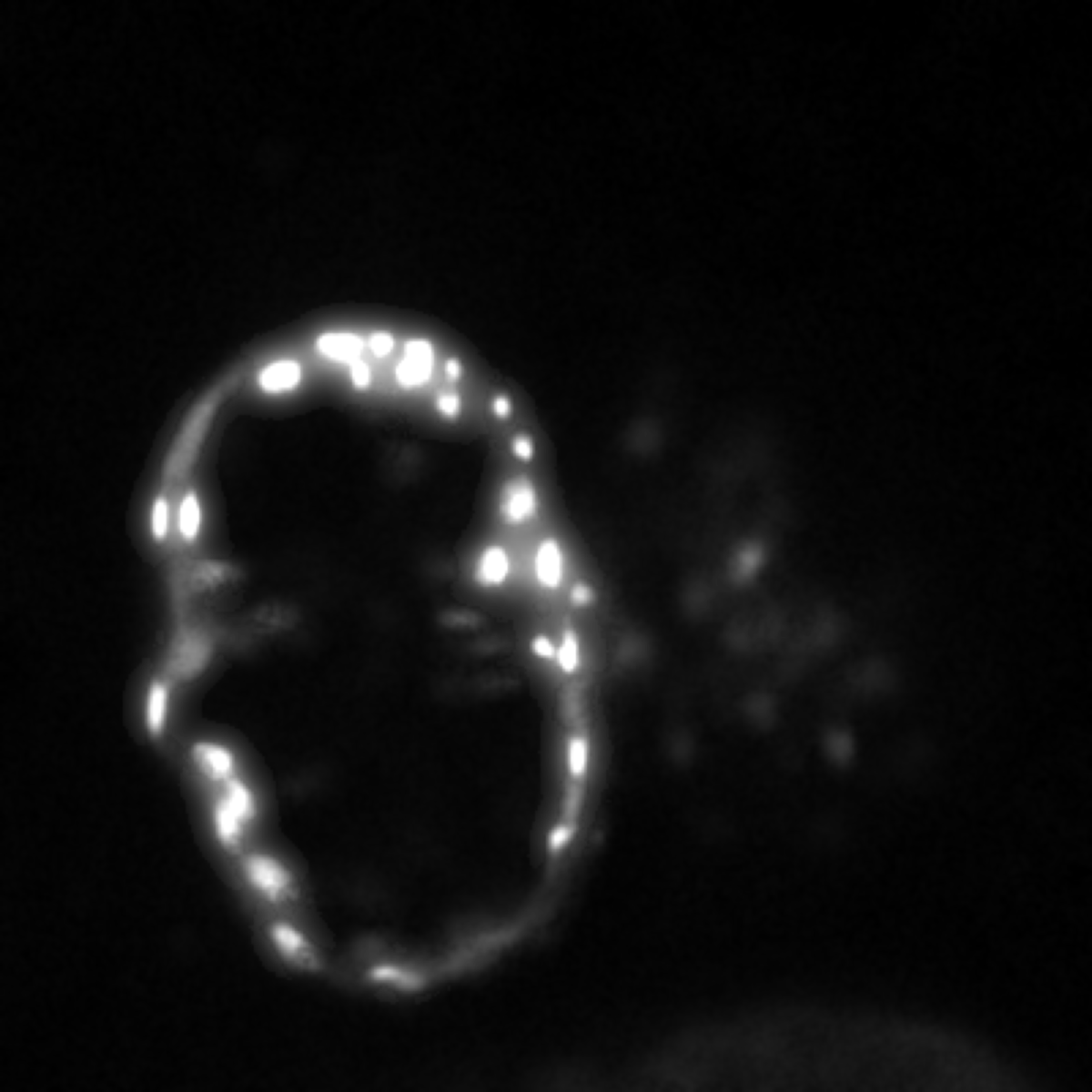} &
\includegraphics[width=0.22\textwidth]{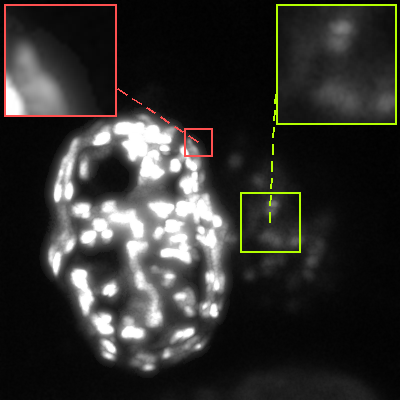} &
\includegraphics[width=0.22\textwidth]{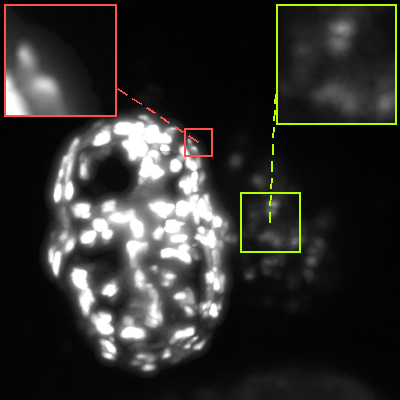} \\[3pt]

\makebox[0pt][r]{\scriptsize\textbf{DnCNN}\hspace{0.6em}}%
\includegraphics[width=0.22\textwidth]{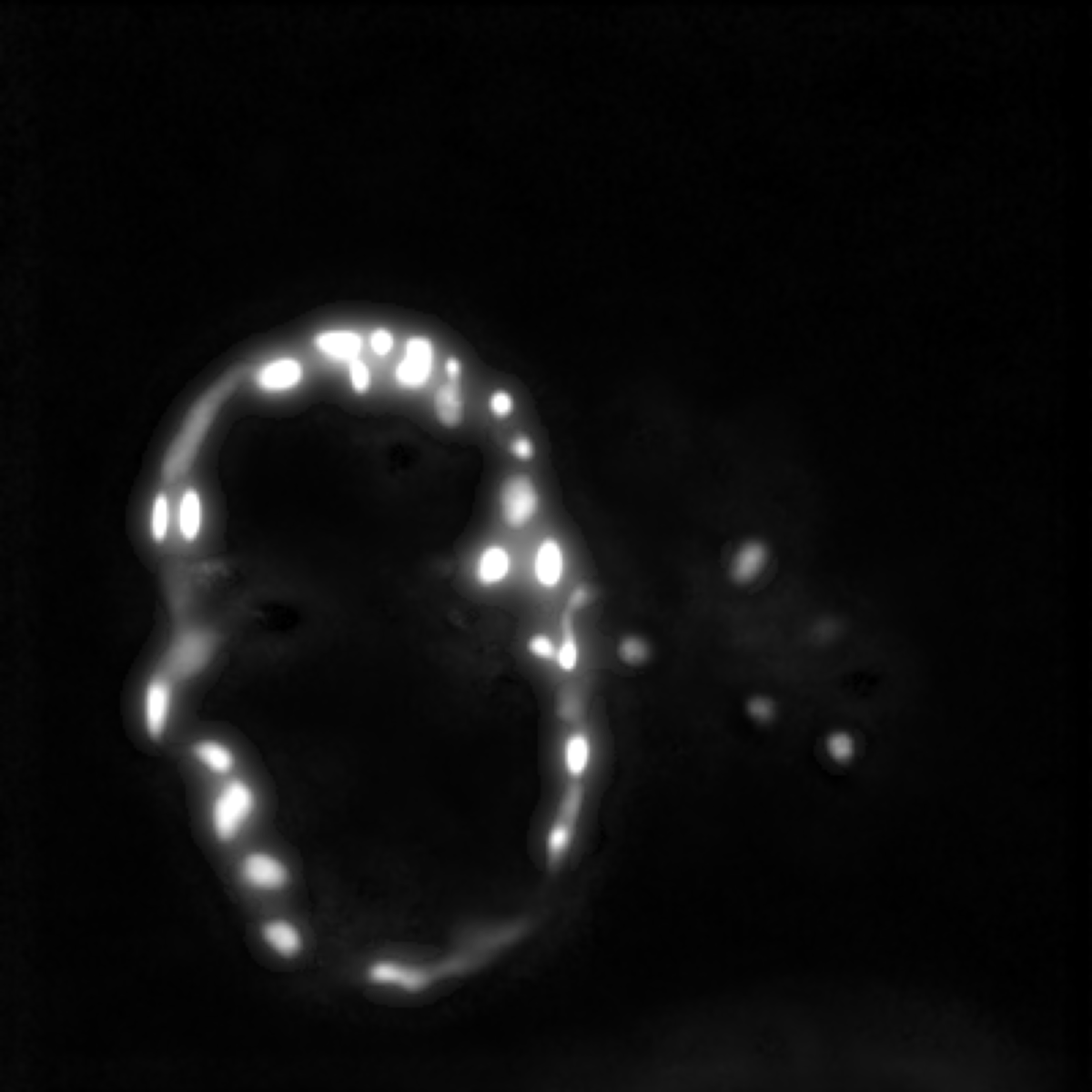} &
\includegraphics[width=0.22\textwidth]{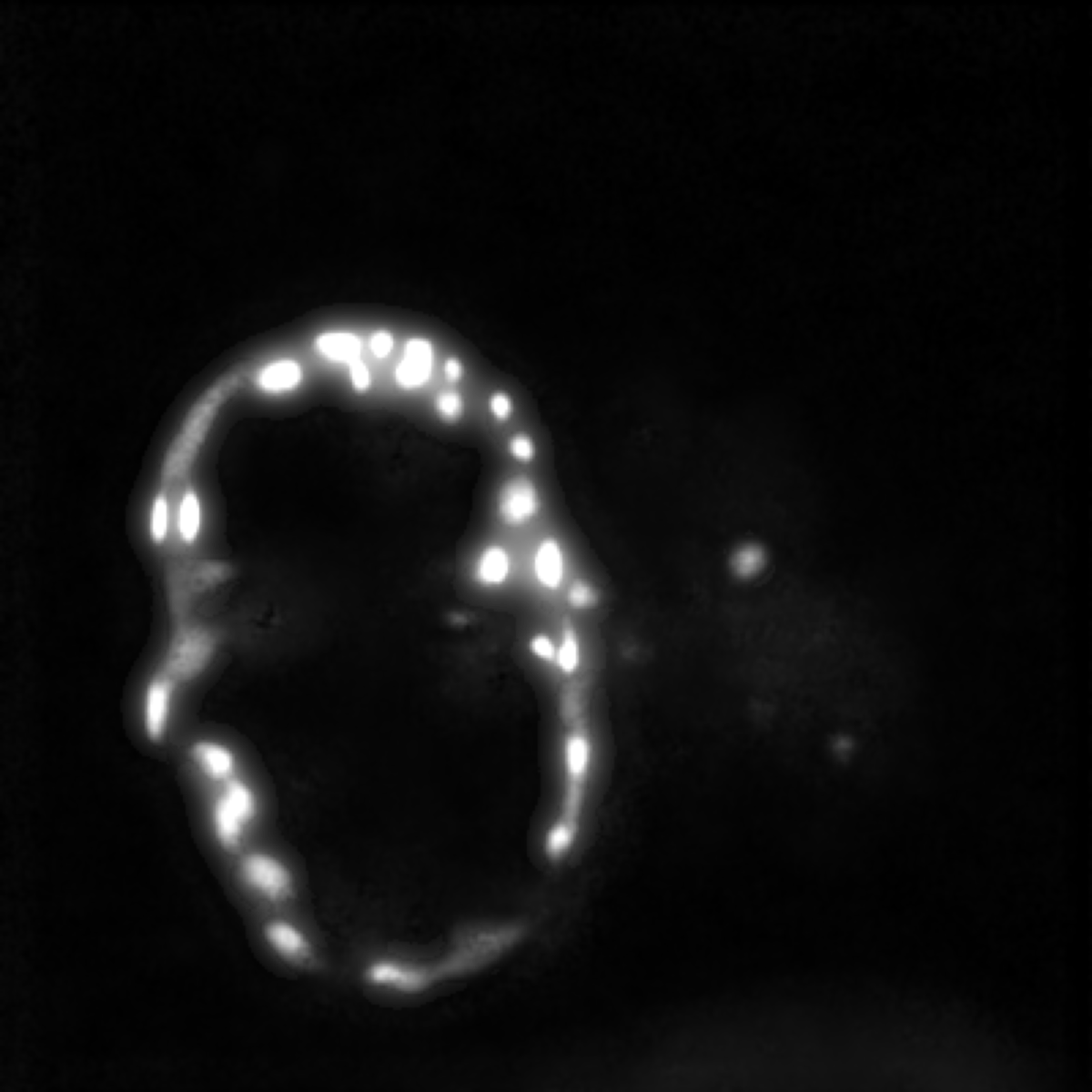} &
\includegraphics[width=0.22\textwidth]{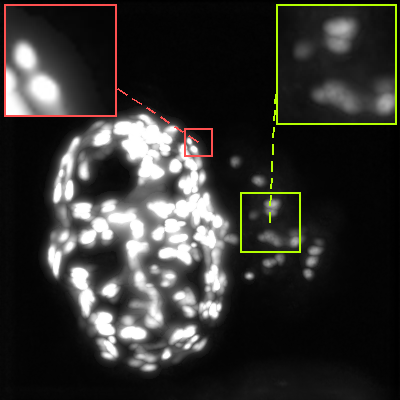} &
\includegraphics[width=0.22\textwidth]{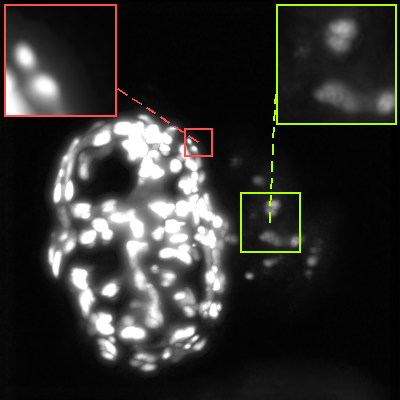} \\[3pt]

\makebox[0pt][r]{\scriptsize\textbf{FFDNet}\hspace{0.6em}}%
\includegraphics[width=0.22\textwidth]{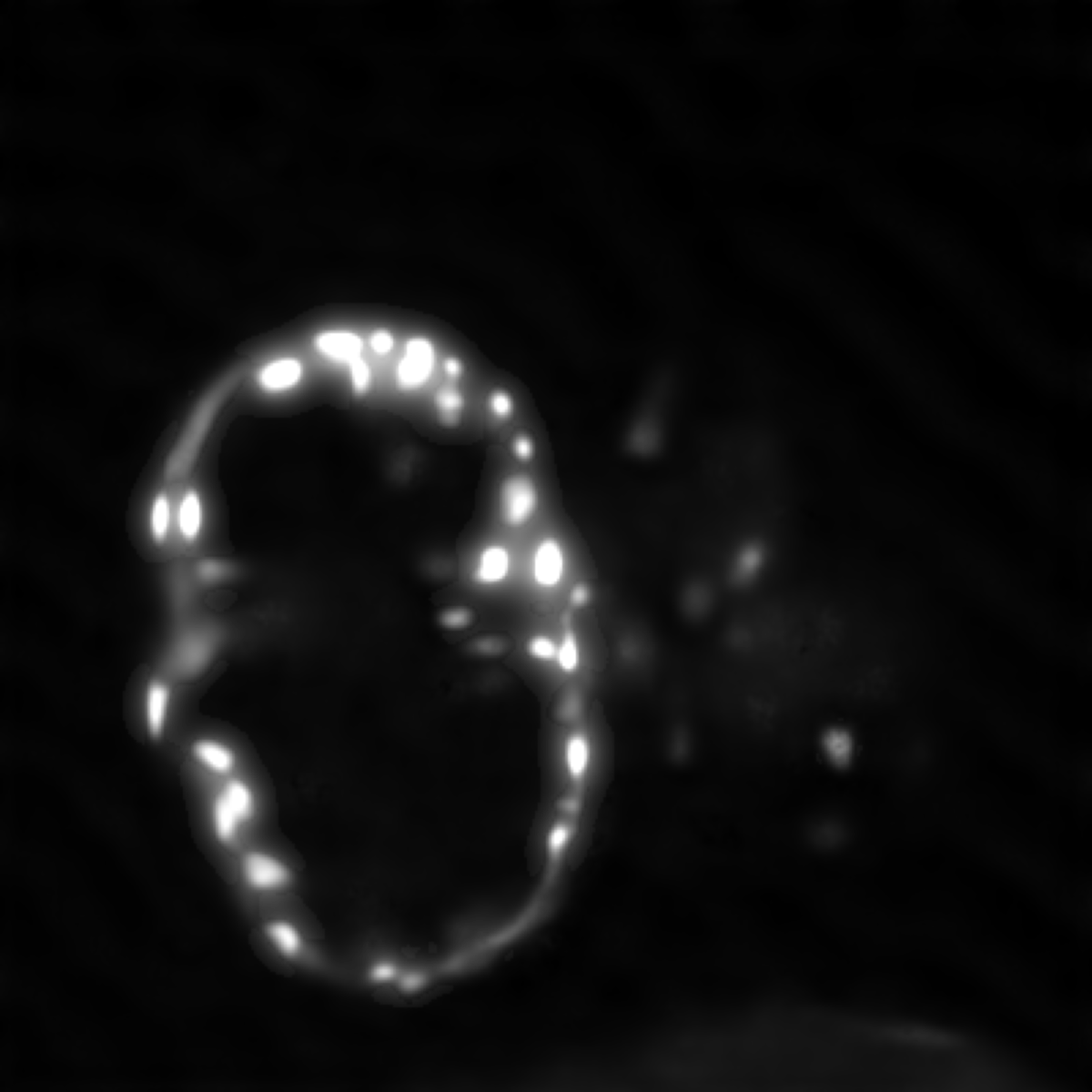} &
\includegraphics[width=0.22\textwidth]{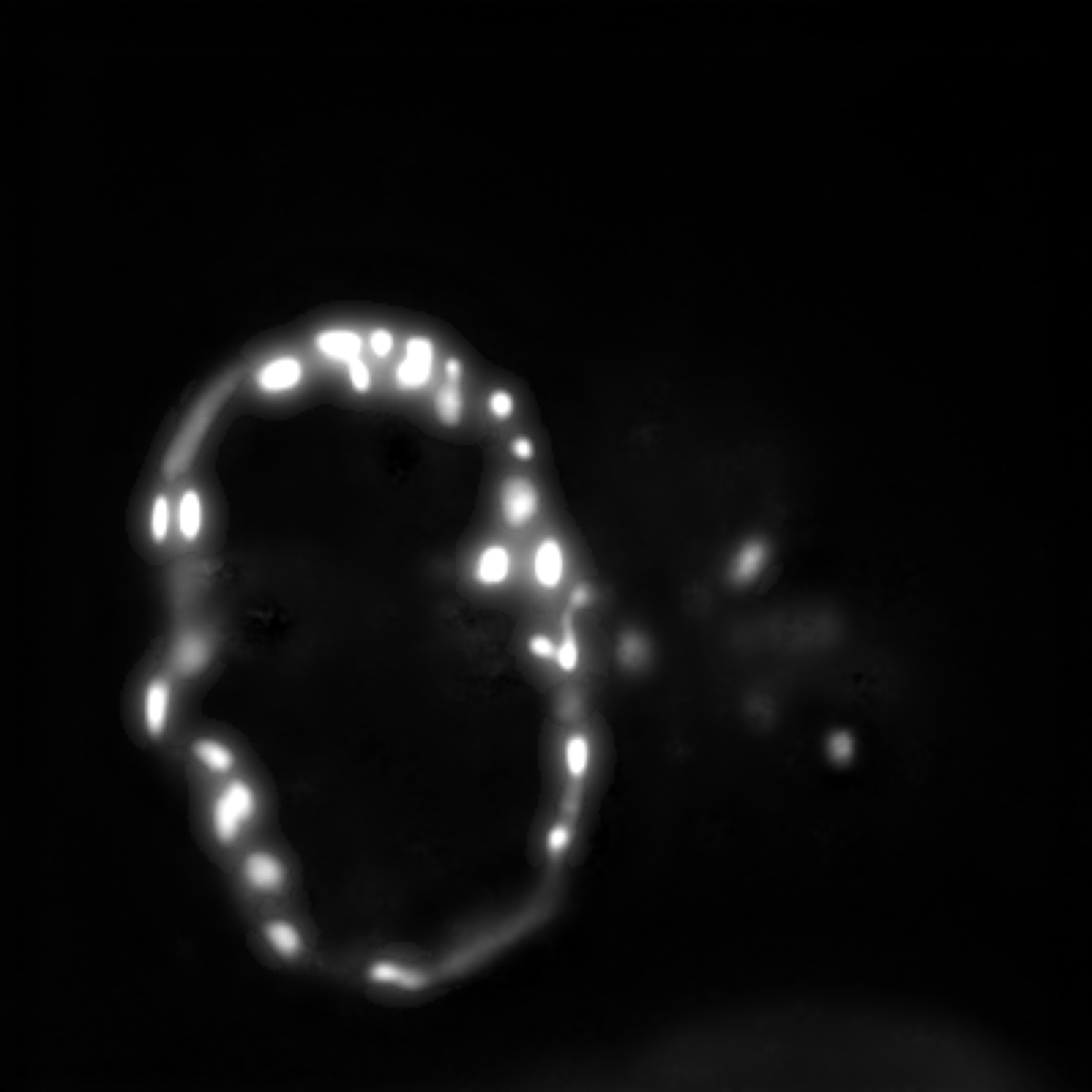} &
\includegraphics[width=0.22\textwidth]{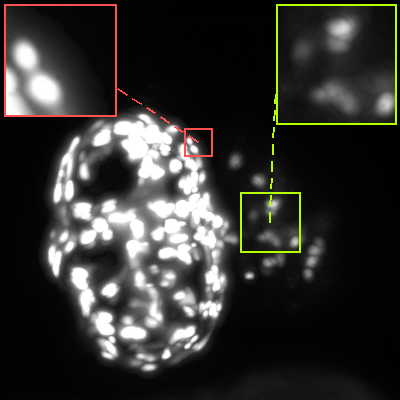} &
\includegraphics[width=0.22\textwidth]{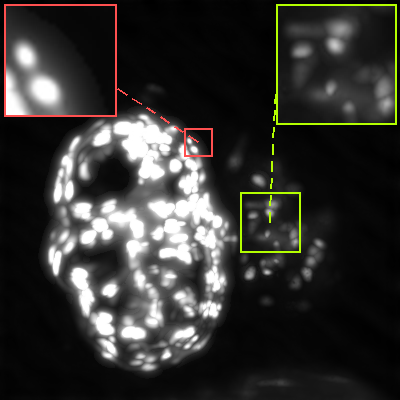} \\[3pt]

\makebox[0pt][r]{\scriptsize\textbf{DRUNet}\hspace{0.6em}}%
\includegraphics[width=0.22\textwidth]{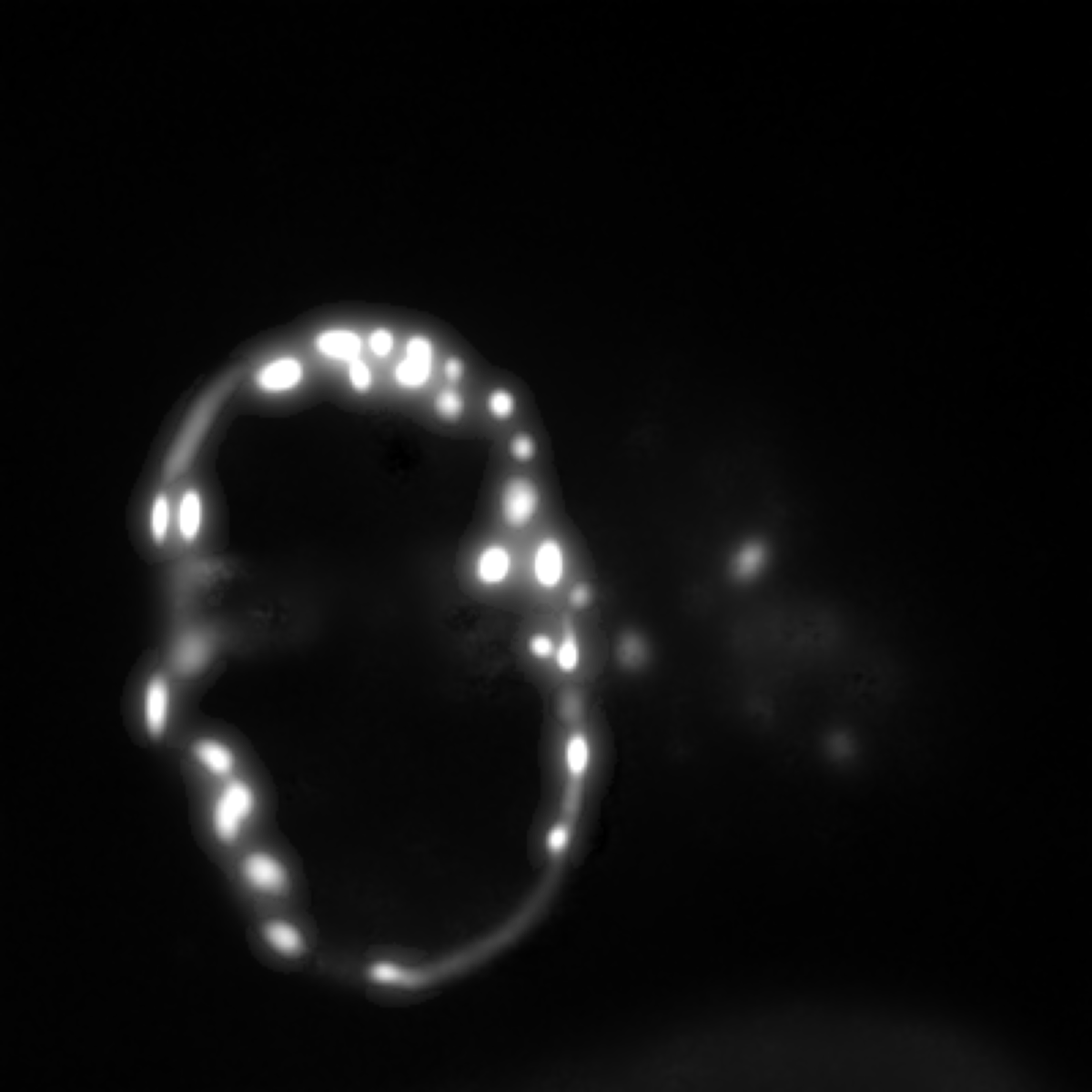} &
\includegraphics[width=0.22\textwidth]{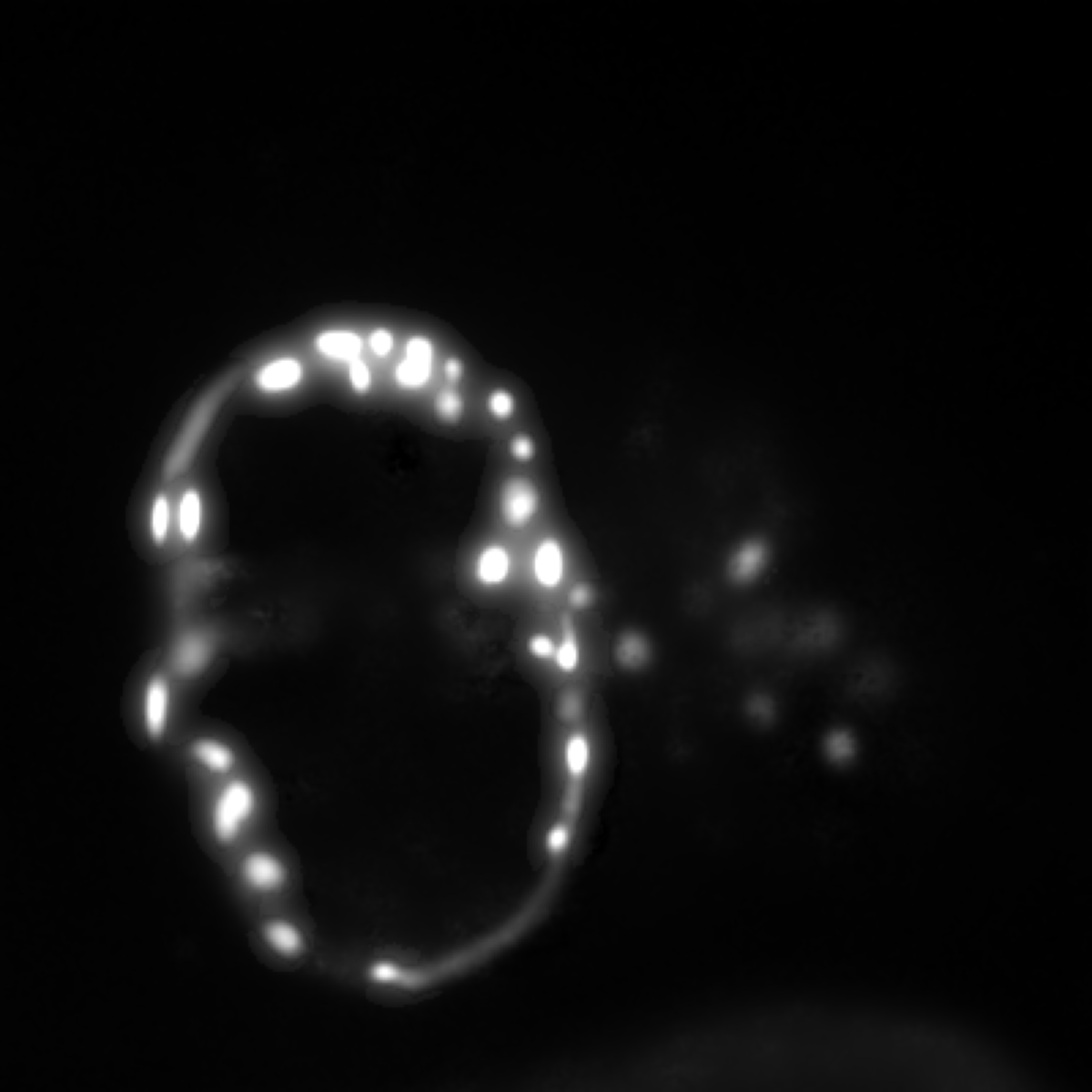} &
\includegraphics[width=0.22\textwidth]{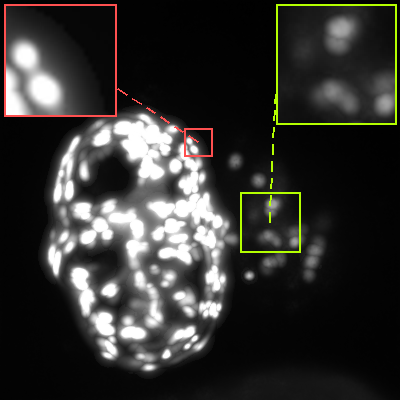} &
\includegraphics[width=0.22\textwidth]{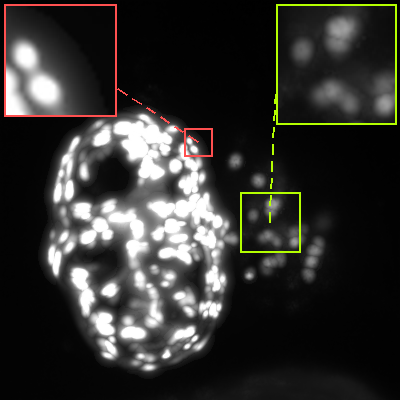} \\

\end{tabular}%
}
\caption{Qualitative comparison on synthetic data using a representative axial slice and MIP. }
\label{fig:synthetic_qualitative_7x4}
\end{figure*}

\begin{figure*}[t]
    \centering
    \captionsetup[subfigure]{labelformat=parens,labelsep=none}

    \begin{subfigure}{0.10\textwidth}
        \centering
        \includegraphics[width=\linewidth]{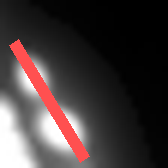}
        \caption{}
        \label{fig:red_profile_location}
    \end{subfigure}
    \hfill
    \begin{subfigure}{0.42\textwidth}
        \centering
        \includegraphics[width=\linewidth]{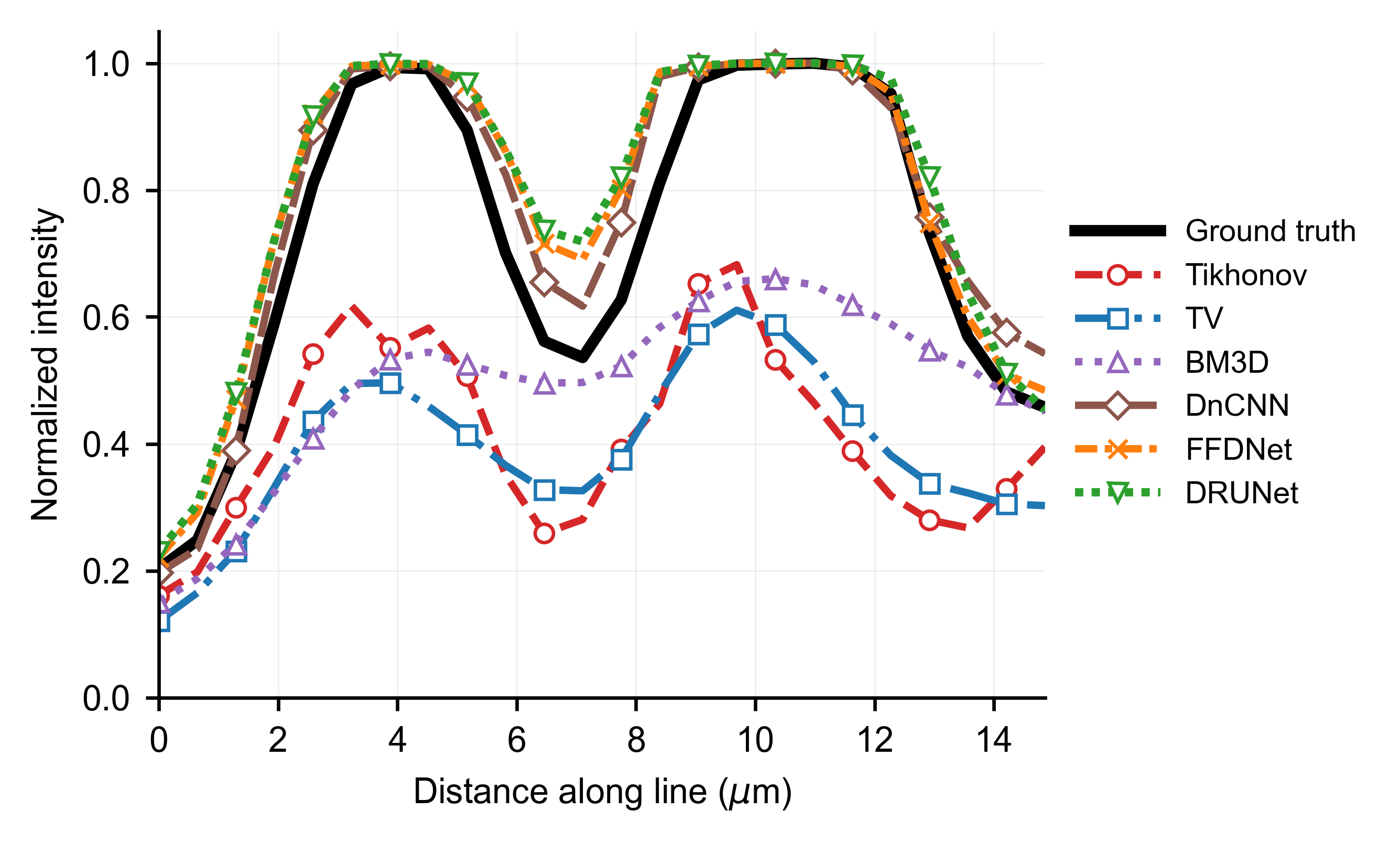}
        \caption{}
        \label{fig:red_profile_slice}
    \end{subfigure}
    \hfill
    \begin{subfigure}{0.42\textwidth}
        \centering
        \includegraphics[width=\linewidth]{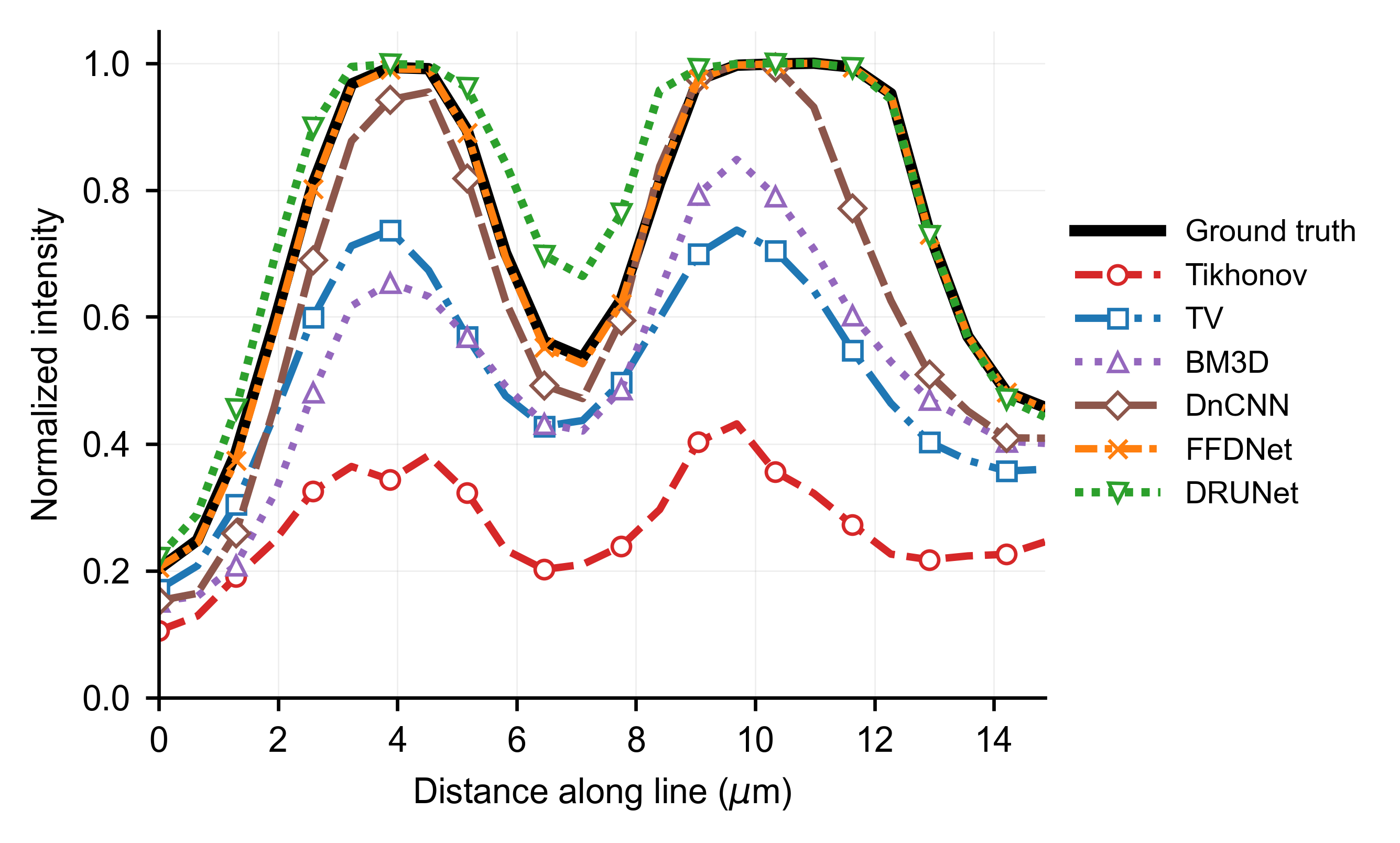}
        \caption{}
        \label{fig:red_profile_axial}
    \end{subfigure}

    \caption{Intensity line profiles from the red zoomed region in Fig.~\ref{fig:synthetic_qualitative_7x4}.  (a) Profile location. 
    (b) Slice-based reconstructions.  (c) Axial-coupled reconstructions.  The profiles show peak separation and local contrast relative to the ground truth. 
    }
    \label{fig:red_line_profile}
\end{figure*}

\paragraph{Qualitative results}

In all qualitative visualizations, the ventricle and atrium of zebrafish-heart are labeled as V and A respectively for anatomical reference, and all scale bars represent 30 $\mu m$. Figure~\ref{fig:synthetic_qualitative_7x4} presents visual reconstruction results on the synthetic dataset.  For each method, we display slice 15 as a representative axial slice, together with the maximum intensity projection (MIP) along the axial direction. MIP is a standard volume-rendering technique that projects the maximum intensity along viewing ray \cite{fishman2006volume}, making it well suited for visualizing bright fluorescent structures in 3D. Here, the slice view assesses local image fidelity at a fixed depth, while the MIP provides a global visual summary of the reconstructed volume, emphasizing volumetric continuity, background suppression, and the recovery of bright cellular or nuclear signals in 3D. We further highlight two zoomed-in regions (red and green) to illustrate key reconstruction challenges: avoiding the merging or distortion of adjacent bright structures and preserving weak signals in low-intensity regions. These features are critical for downstream tasks, including cell detection, segmentation, and 3D/4D tracking in zebrafish-heart imaging. Undetected weak signals, merged neighboring structures, and spurious background artifacts can bias the quantitative analyses of cardiac morphology and contractile motion \cite{zhang2024_4d,saberigarakani2025volumetric}.

The main visual differences across denoisers lie in fine-structure recovery and background cleanliness. Tikhonov recovers the coarse cardiac shape but leaves a noisy background in the slice-based reconstruction and suppresses many weak cellular or nuclear signals. TV gives a cleaner and sharper baseline, but it still oversmooths weak signals and thin structures. BM3D preserves more local detail than TV, although residual low-intensity artifacts remain visible. The DL-based denoisers recover brighter and better-separated cellular or nuclear signals, with FFDNet and DRUNet appearing closest to the ground truth. 
This is especially visible in the red zoomed region, where FFDNet and DRUNet better preserve the separation and morphology of neighboring bright cellular or nuclear signals. To further quantify this local comparison, Fig.~\ref{fig:red_line_profile} shows intensity profiles as a function of spatial position along a representative oblique line in the red zoomed region. The profiles are plotted against physical distance in $\mu$m and globally normalized using the intensity range across all compared images. FFDNet and DRUNet more closely match the ground-truth peak locations and valley structure, whereas Tikhonov and TV tend to suppress peak amplitudes and reduce local contrast.

The effect of axial coupling is more evident in the MIP views, where slice-to-slice inconsistencies accumulate across the axial stack. Compared with the slice-based reconstructions, the axial-coupled results generally show cleaner backgrounds and cardiac boundaries, and more stable recovery of cellular structures across $z$. 
The improvement of axial coupling appears in different local regions for different denoisers. For example, BM3D shows improved separation and morphology of adjacent bright signals in the red zoomed region, while FFDNet recovers weak signals in the green zoomed region closer to the ground truth. 
Comparing Figures~\ref{fig:red_profile_slice} and~\ref{fig:red_profile_axial}, the axial-coupled profiles are generally closer to the ground truth for the stronger denoisers, especially in preserving the two-peak structure and the intervening valley. 

Overall, the qualitative results agree with the quantitative comparison in Table~\ref{tab:synthetic_quantitative_main}. FFDNet and DRUNet provide the strongest visual reconstructions, while axial coupling improves the consistency of the reconstructed volume, especially for weak cellular or nuclear signals and for classical priors such as Tikhonov and TV.
Together, these results highlight that the choice of denoiser determines the baseline reconstruction quality, while axial coupling provides additional benefit by leveraging inter-slice correlations, particularly when the slice-based prior is relatively weak.

\subsection{Real data}
\begin{figure*}
\centering
\setlength{\tabcolsep}{2pt}
\renewcommand{\arraystretch}{1.02}

\begin{tabular}{@{}c c c@{}}
\textbf{Tikhonov} & \textbf{TV} & \textbf{BM3D} \\[3pt]

\makebox[0pt][r]{\scriptsize\textbf{Slice-based}\hspace{0.6em}}%
\includegraphics[width=0.26\textwidth]{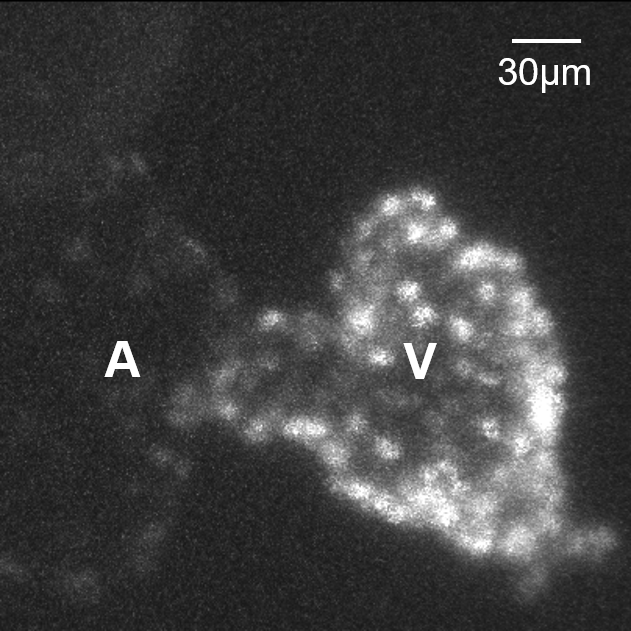} &
\includegraphics[width=0.26\textwidth]{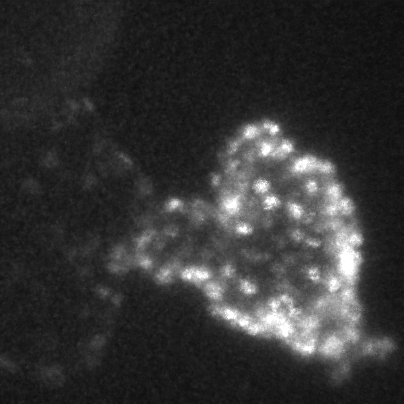} &
\includegraphics[width=0.26\textwidth]{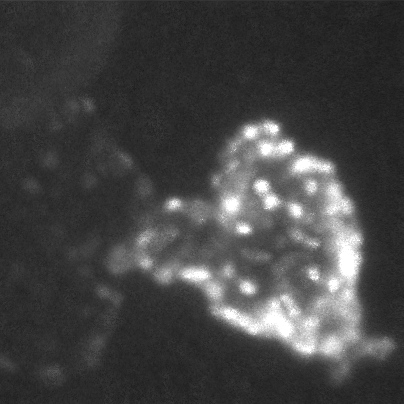} \\[2pt]

\makebox[0pt][r]{\scriptsize\textbf{Axial-coupled}\hspace{0.6em}}%
\includegraphics[width=0.26\textwidth]{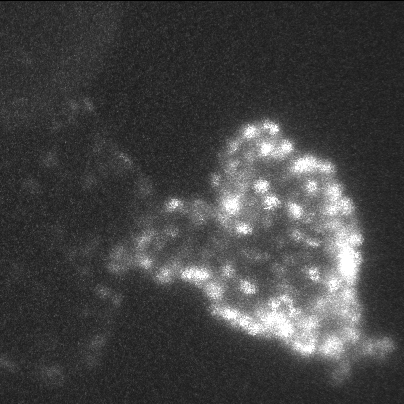} &
\includegraphics[width=0.26\textwidth]{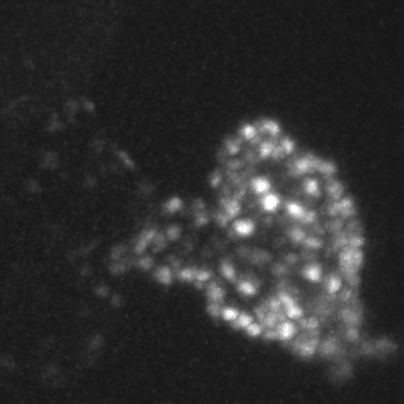} &
\includegraphics[width=0.26\textwidth]{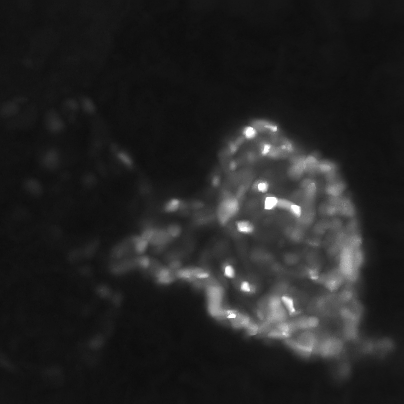} \\[6pt]

\textbf{DnCNN} & \textbf{FFDNet} & \textbf{DRUNet} \\[3pt]

\makebox[0pt][r]{\scriptsize\textbf{Slice-based}\hspace{0.6em}}%
\includegraphics[width=0.26\textwidth]{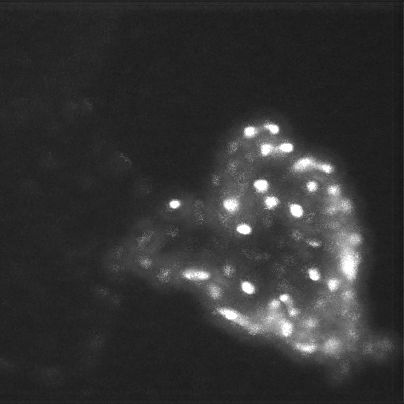} &
\includegraphics[width=0.26\textwidth]{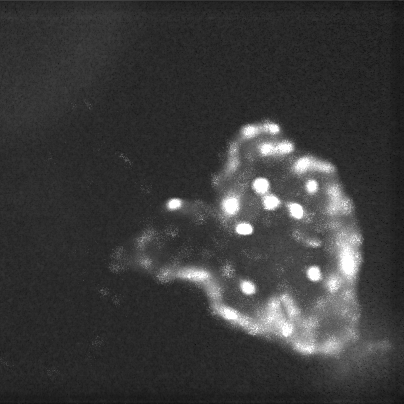} &
\includegraphics[width=0.26\textwidth]{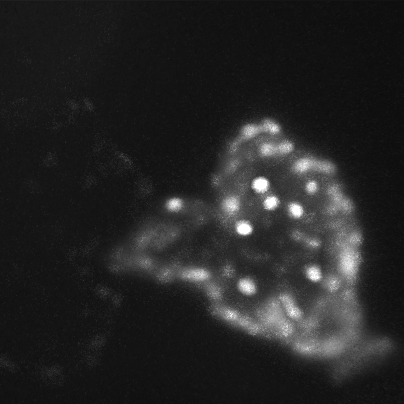} \\[2pt]

\makebox[0pt][r]{\scriptsize\textbf{Axial-coupled}\hspace{0.6em}}%
\includegraphics[width=0.26\textwidth]{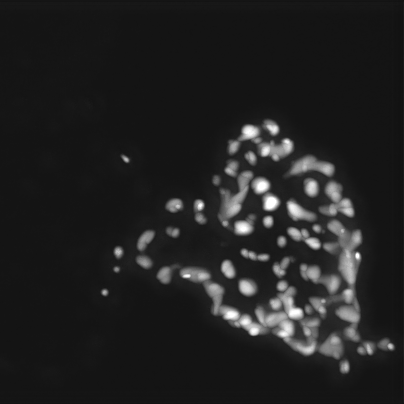} &
\includegraphics[width=0.26\textwidth]{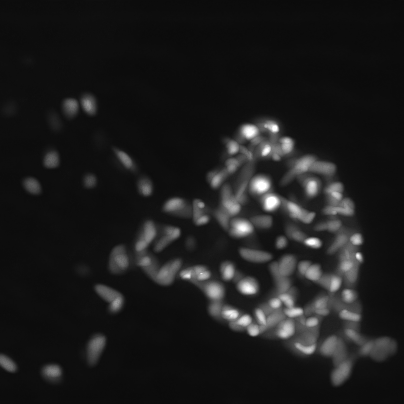} &
\includegraphics[width=0.26\textwidth]{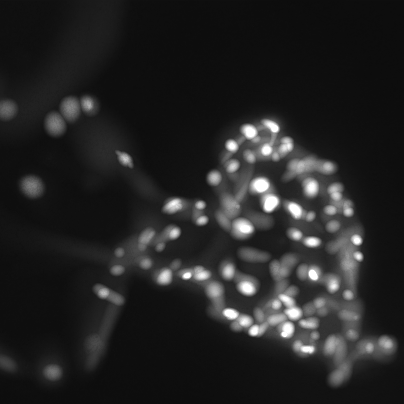}
\end{tabular}

\caption{Qualitative real-data comparison using MIPs.}
\label{fig:real_visual_main}
\end{figure*}

Figure~\ref{fig:real_visual_main} shows real-data reconstruction results using MIPs. Because the real-data evaluation is qualitative, we focus on structural continuity, local contrast, background suppression, and the visibility of cellular or nuclear signals rather than voxel-wise fidelity. The comparison examines practical trade-offs among different denoisers, specifically in terms of weak-signal preservation, background suppression, and the recovery of plausible cardiac morphology.

Among the classical priors, Tikhonov and TV retain the main cardiac region and many diffuse low-intensity structures, but at the cost of noticeable background haze and relatively low local contrast. BM3D reduces background variation and yields a cleaner appearance, though some weak structures may also be attenuated in the process. The DL-based denoisers produce sharper bright signals and higher local contrast, but they tend to emphasize the most prominent structures while suppressing some diffuse weak signals. Among them, FFDNet and DRUNet provide the clearest separation between bright nuclear signals and the background.

Axial coupling primarily improves the real-data MIPs by reducing background haze and enhancing the visibility of bright structures. This effect is most pronounced for the DL-based denoisers, where the axial-coupled reconstructions exhibit sharper and more distinctly separated cellular or nuclear signals compared with slice-based results. However, without a matched ground truth, sharper punctate signals alone do not establish improved reconstruction fidelity; they may reflect genuine contrast improvement, denoiser-induced enhancement, or both. Their practical relevance therefore depends on the downstream task.

The real-data experiment presents additional challenges: the effective sensing operator may deviate from the ideal forward model due to mask misalignment, nonideal modulation, or calibration errors, and the noise statistics may not exactly match the assumed model. Without ground truth, the comparison instead reveals practical visual trade-offs: classical priors tend to retain diffuse low-intensity content but leave stronger background haze, while DL-based denoisers improve local contrast and background suppression at the cost of reduced visibility of weaker diffuse signals. This trade-off highlights a practical advantage of the PnP formulation, namely that the reconstruction prior can be selected according to the needs of the downstream task.

\subsection{Effect of compression ratio}\label{sect:exp4CR}

We investigate the dependence of reconstruction quality on the compression ratio $R$, defined as the number of axial slices multiplexed into one camera exposure. Larger values of $R$ improve acquisition efficiency by reducing the number of measurements per volume, but also increase the degree of ill-posedness of the inverse problem.

Figure~\ref{fig:cr_study_main} reports PSNR and SSIM on the synthetic dataset for TV and FFDNet, representing classical and DL-based denoisers, respectively. For both methods, reconstruction quality generally decreases as $R$ increases because each reconstructed volume is supported by fewer measurements. For TV, the axial-coupled model consistently outperforms the slice-based model across the tested compression ratios, especially in SSIM. 
This indicates that inter-slice correlation provides useful structural support when the slice-wise prior is relatively simple. For FFDNet, the slice-based reconstruction is already strong at low compression ratios, so the additional gain from axial coupling is smaller and not uniformly positive across all $R$. However, at higher compression ratios like $R=10$, the axial-coupled FFDNet reconstruction improves both PSNR and SSIM over slice-based result. 

\begin{figure}
    \centering
    \begin{subfigure}{0.48\linewidth}
        \centering
        \includegraphics[width=\linewidth]{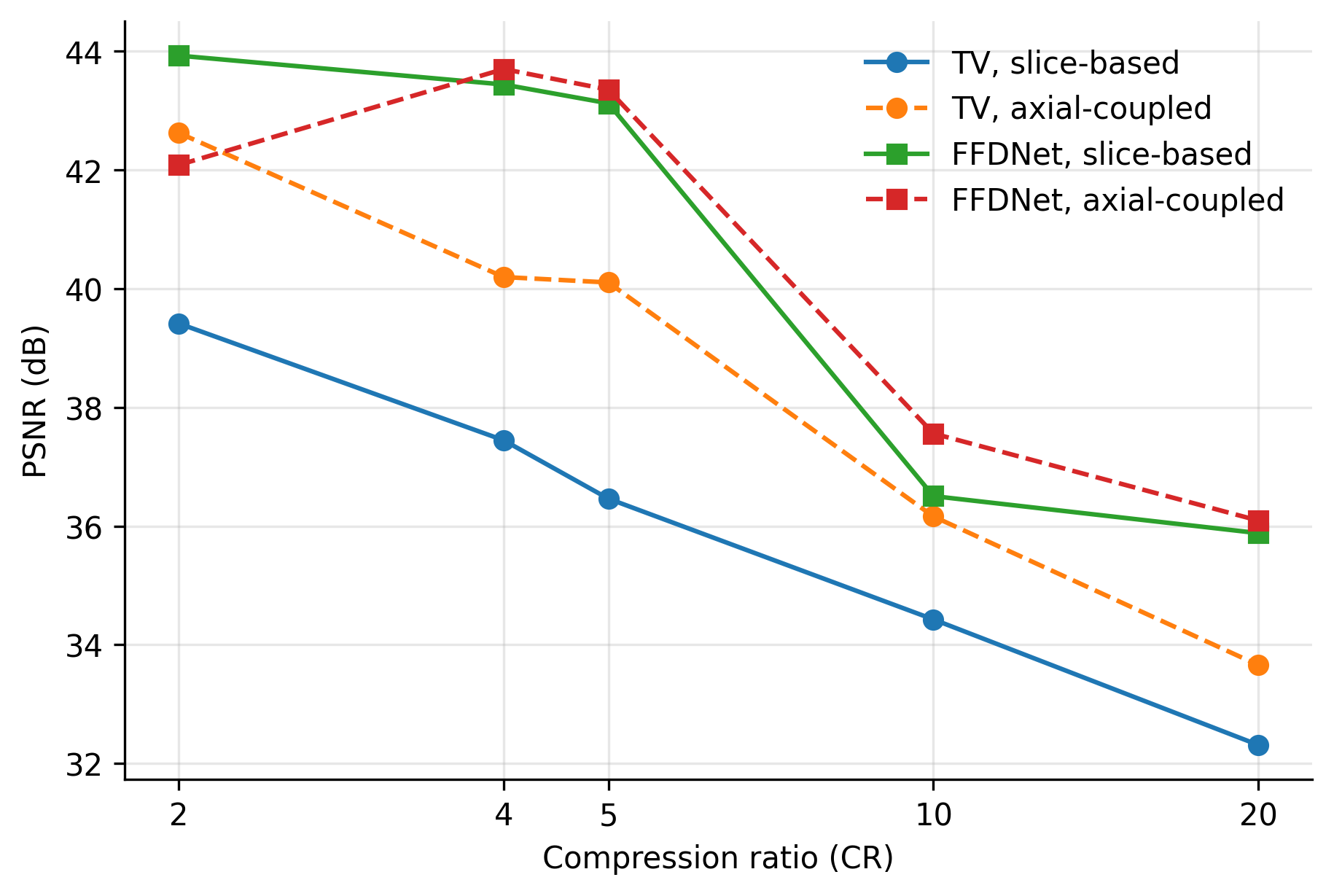}
        \caption{PSNR}
        \label{fig:cr_psnr}
    \end{subfigure}
    \hfill
    \begin{subfigure}{0.48\linewidth}
        \centering
        \includegraphics[width=\linewidth]{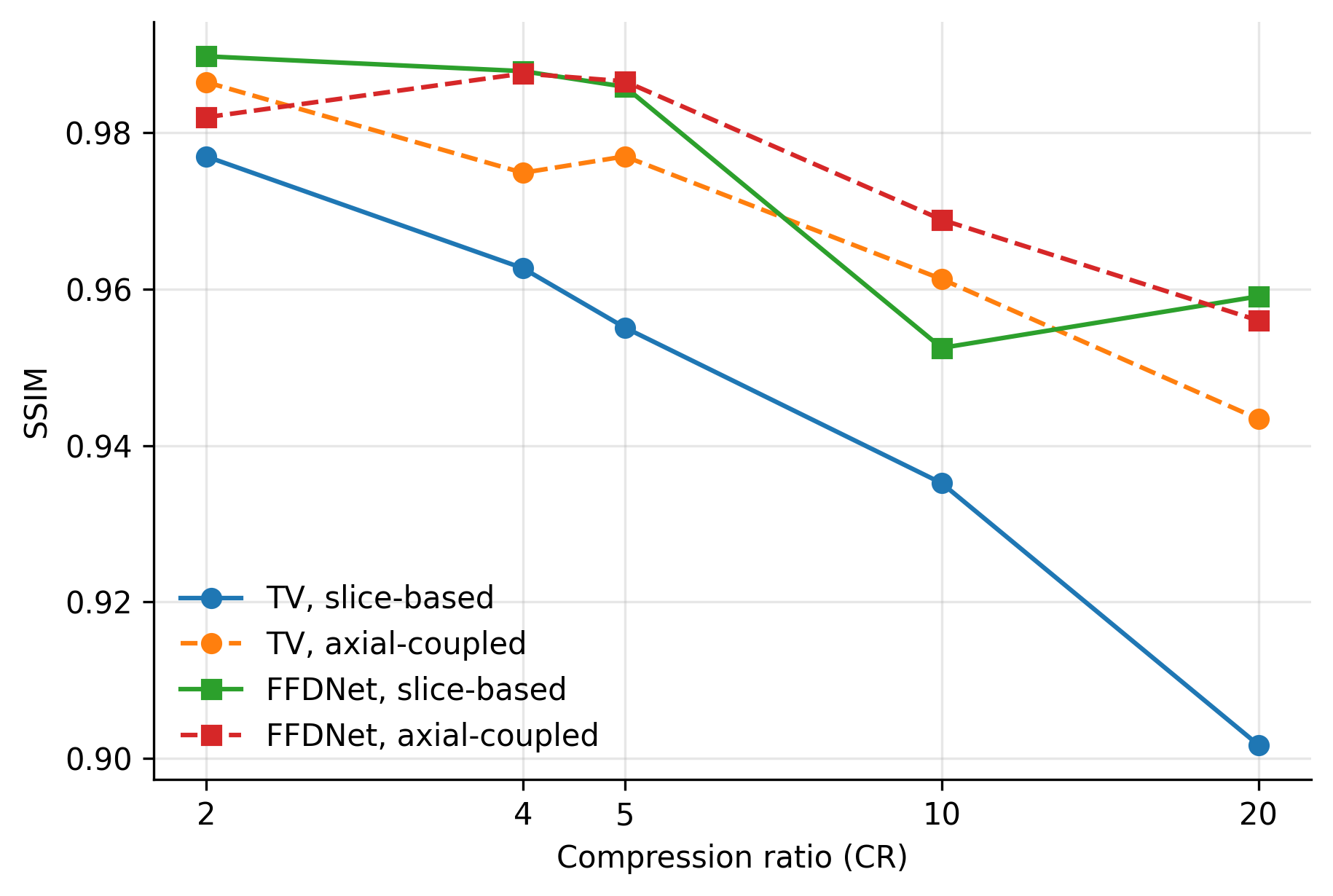}
        \caption{SSIM}
        \label{fig:cr_ssim}
    \end{subfigure}
    \caption{Effect of compression ratio $R$ on synthetic reconstruction quality for TV and FFDNet. Each curve reports slice-averaged PSNR or SSIM for the slice-based and axial-coupled models.}
    \label{fig:cr_study_main}
\end{figure}

Overall, the choice of $R$ should be guided by the downstream task. When the analysis relies on weak cellular or nuclear signals, separation of neighboring cells, or accurate 3D/4D tracking, a smaller $R$ is preferable to preserve fine structural details. 
When acquisition efficiency is more important and the task can tolerate some loss of fine-scale information, a larger $R$ may be acceptable. The axial-coupled model does not remove this acquisition-reconstruction trade-off, but it can partially offset the loss of measurement information at stronger compression.

\section{Conclusion and future work} \label{sec:conclusion}

We proposed a PnP-ADMM framework for volumetric reconstruction in CS-LSM. The method handles binary-mask encoded axial measurements and flexibly incorporates any off-the-shelf denoiser, encompassing classical and DL-based methods alike. We developed a slice-based approach as a tractable baseline, while the axial-coupled extension exploits inter-slice continuity to improve reconstruction quality across the volume. Furthermore, we established subsequential convergence of the proposed scheme under a weakly convex regularization assumption, with every accumulation point shown to be a stationary point of the corresponding variational formulation. Experiments on zebrafish-heart data show that the proposed framework can recover cellular structures from compressed measurements while accommodating different denoising methods. The axial-coupled model generally improves reconstruction quality and volumetric continuity over the slice-based model. Future work will investigate tuning-free approaches based on deep unrolling \cite{chen2026deep} and reinforcement learning \cite{wei2020tuning}, as well as extend the proposed framework to 4D spatiotemporal reconstruction. On the theoretical side, establishing convergence guarantees for the proposed axial-coupled PnP framework, beyond the proximal-based denoiser regime, remains an important open problem.

\bibliographystyle{siamplain}
\bibliography{ref}

\end{document}